\pdfoutput=1

\documentclass[11pt]{article}

\usepackage[final]{acl}

\usepackage{times}
\usepackage{latexsym}

\usepackage[T1]{fontenc}

\usepackage[utf8]{inputenc}

\usepackage{microtype}
\usepackage{amsthm}

\newtheorem{example}{Example}
\usepackage{inconsolata}

\usepackage{graphicx}

\newcommand{\cw}{\mathcal{W}}
\newcommand{\be}{\mathbb{E}}
\newcommand{\bp}{\mathbb{P}}
\newcommand{\norm}[1]{\lVert #1 \rVert}

\usepackage{colortbl}

\usepackage{algorithm}
\usepackage{algorithmic}
\usepackage{caption}
\usepackage{amsmath}
\usepackage{amssymb}
\usepackage{mathtools}
\usepackage{amsthm}
\usepackage{adjustbox}

\usepackage[capitalize,noabbrev]{cleveref}

\theoremstyle{plain}
\newtheorem{theorem}{Theorem}[section]

\newtheorem{lemma}[theorem]{Lemma}

\theoremstyle{definition}
\newtheorem{definition}[theorem]{Definition}

\theoremstyle{remark}

\usepackage[textsize=tiny]{todonotes}

\usepackage{amsmath,amsfonts,bm}

\def\eqref#1{equation~\ref{#1}}

\def\1{\bm{1}}

\def\vtheta{{\bm{\theta}}}

\DeclareMathAlphabet{\mathsfit}{\encodingdefault}{\sfdefault}{m}{sl}
\SetMathAlphabet{\mathsfit}{bold}{\encodingdefault}{\sfdefault}{bx}{n}

\def\sB{{\mathcal{B}}}

\def\sD{{\mathcal{D}}}
\def\sE{{\mathcal{E}}}

\def\sL{{\mathcal{L}}}

\def\sP{{\mathcal{P}}}

\def\sT{{\mathcal{T}}}

\def\sX{{\mathcal{X}}}
\def\sY{{\mathcal{Y}}}

\newcommand{\E}{\mathbb{E}}

\DeclareMathOperator*{\argmin}{arg\,min}

\usepackage{soul}
\usepackage{xcolor}

\def\valpha{\bm{\alpha}}

\usepackage{xfrac}

\usepackage[toc,page,header]{appendix}
\usepackage{minitoc}

\title{AutoMixAlign: Adaptive Data Mixing for Multi-Task Preference Optimization in LLMs}

\author{
 \textbf{Nicholas E. Corrado\textsuperscript{1,2}},
 \textbf{Julian Katz-Samuels\textsuperscript{2}},
 \textbf{Adithya Devraj\textsuperscript{2}},
 \textbf{Hyokun Yun\textsuperscript{2}},
\\
 \textbf{Chao Zhang\textsuperscript{2,3}},
 \textbf{Yi Xu\textsuperscript{2}},
 \textbf{Yi Pan\textsuperscript{2}},
 \textbf{Bing Yin\textsuperscript{2}},
 \textbf{Trishul Chilimbi\textsuperscript{2}}
\\
\\
 \textsuperscript{1}University of Wisconsin--Madison,
 \textsuperscript{2}Amazon,
 \textsuperscript{3}Georgia Institute of Technology
\\
 \small{
   \textbf{Correspondence:} \href{mailto:email@domain}{ncorrado@wisc.edu, jkatzsamuels@gmail.com}
 }
}

\begin{document}
\doparttoc 
\faketableofcontents 

\maketitle
\begin{abstract}
When aligning large language models (LLMs), their performance on various tasks (such as being helpful, harmless, and honest) depends heavily on the composition of their training data.
However, selecting a data mixture that achieves strong performance across all tasks is challenging. 
Existing approaches rely on large ablation studies, heuristics, or human intuition, but these can be prohibitively expensive and suboptimal.
We study this problem in the setting of preference optimization via DPO and introduce AutoMixAlign (AMA), a theoretically-grounded algorithm that adaptively mixes datasets during training to balance performance across tasks.
AMA first trains \textit{specialist models} for each task to determine losses that correspond to strong task performance.
Then, it trains a generalist model using a novel minimax optimization that prioritizes tasks for which generalist model losses deviate most from specialist model losses.
To optimize this problem, we propose two algorithms: (1) AMA-R, which adaptively reweights the objective to prioritize tasks, and (2) AMA-S, which adaptively adjusts how much data is sampled from each task to prioritize tasks.
Both algorithms achieve a convergence rate of $O(1/\sqrt{T})$ in the convex case. AMA-R's convergence result follows from~\citet{sagawa2019distributionally}, and we provide a convergence proof for AMA-S using online learning techniques such as EXP3~\cite{auer2002finite}.
We evaluate AMA on several multitask alignment setups and find that AMA outperforms the standard alignment approach---which simply optimizes the total loss across all tasks---and also outperforms model merging methods.
\end{abstract}


\section{Introduction}

\begin{figure*}
    \centering
    \includegraphics[width=\linewidth]{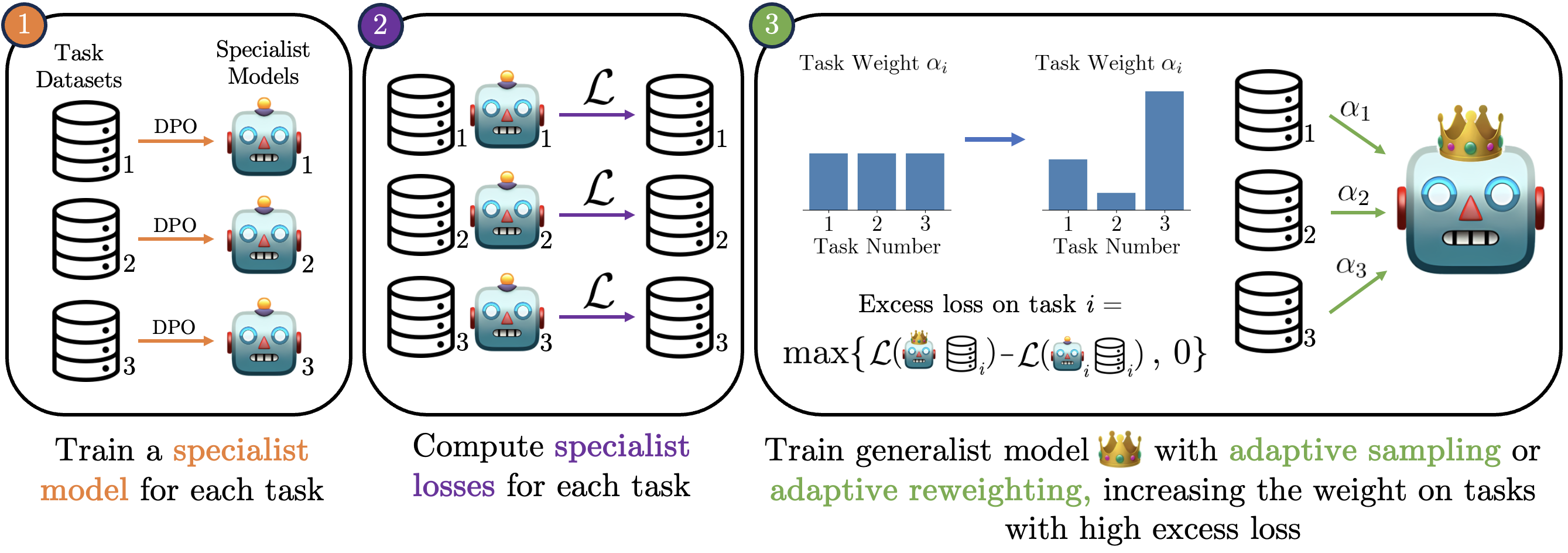}
    \caption{\textbf{Overview of AMA}. First, we train  \textit{specialist models} on each task dataset individually to obtain a model that performs well on that specific task. 
    Next, we pre-compute the losses achieved by each specialist model and add these losses to their respective datasets.
    Lastly, we use a minimax optimization algorithm to optimize a generalist model towards the losses achieved by the specialist models.}
    \label{fig:overview}
\end{figure*}

As LLMs are increasingly deployed in real-world applications, it has become essential to develop generalist models with strong capabilities across a broad range of tasks (\textit{e.g.} math, coding, safety). 
A key step toward this objective is preference optimization, learning from human or synthetic preference data with notable methods like RLHF~\cite{ouyang2022training} or DPO~\cite{rafailov2023direct}.
To produce generalist models, LLMs are typically trained on a diverse mix of datasets, each targeting different tasks. 
However, the training data composition can significantly affect performance in each task~\citep{brown2020language, du2022glam, chowdhery2023palm, ivison2024UnpackingDA}. This observation motivates the central question in this work: \textit{How should we mix each task dataset to produce a model that performs well across a variety of tasks?}

Currently, practitioners tackle this data mixing problem using heuristics or large-scale ablation studies on dataset mixtures \cite{touvron2023llama, dubey2024llama, ivison2023camels, lambert2024t}. 
Both approaches have significant drawbacks: heuristics tend to produce suboptimal models while ablation studies become prohibitively expensive due to the combinatorial nature of dataset mixing. This problem is particularly acute in production settings, where regular model releases must accommodate evolving datasets and product requirements, creating significant ongoing expenses in computation and human effort.

In this paper, we introduce a novel and theoretically justified approach to the data mixing problem in DPO-based preference optimization.
Our algorithm, AutoMixAlign (AMA), automatically mixes datasets during LLM training to balance performance across multiple tasks. 
AMA has two phases. 
First, AMA trains \textit{specialist models} on each task dataset individually, yielding a strong specialized model for each task.
Second, AMA trains a generalist model to match the specialist model performance by optimizing a novel minimax optimization problem that prioritizes more challenging tasks---tasks with large \emph{excess loss}~\cite{he2024robust, xie2024doremi}, defined as the difference between the  generalist model's losses and the specialist models' losses. We present two algorithms for this optimization: a reweighting algorithm that adaptively reweights tasks in the objective, and a resampling algorithm that adaptively adjusts how much data is sampled from each task.

Both algorithms enjoy theoretical guarantees. The reweighting algorithm has an $O({1}/{\sqrt{T}})$ convergence rate in the convex setting, following immediately from~\citet{sagawa2019distributionally}. For the resampling algorithm, we prove convergence by modeling it as a two-player zero-sum game in which the learner aims to minimize the excess losses, while the adversary shifts the mixture distribution to sample more from tasks with larger excess losses. Using online learning techniques such as EXP3~\cite{auer2002finite}, we prove that the resampling algorithm has an $O({1}/{\sqrt{T}})$ convergence rate in the convex case. 


We conduct extensive experiments demonstrating that AMA effectively balances helpfulness, harmlessness, and reasoning tasks.
In all experiments, AMA outperforms standard baselines such as uniform data mixing and model merging, which often perform poorly on one or more tasks.
AMA improves average performance over standard DPO with uniform data mixing by up to 9.42\% while maintaining robust performance across each task.
%
In summary, we make the following contributions: 
\begin{enumerate}
    \item We introduce AMA, a multi-task learning algorithm that adaptively prioritizes more challenging tasks via reweighting or resampling.
    \item We prove that the resampling-based AMA algorithm has a $O({1}/{\sqrt{T}})$ convergence rate in the convex case. 
    \item We show that AMA has strong empirical performance, improving over standard data mixing approaches by up to 9.42\%.
\end{enumerate}

\section{The Setup}

We assume we have $k$ preference datasets, one for each task $i$. Each dataset is a collection of triples: $\sD_i = \{(x^{(i)}_{1}, y^{(i)}_{1, +}, y^{(i)}_{1, -}), \ldots, (x^{(i)}_n, y^{(i)}_{n,+}, y^{(i)}_{n,-})\}$ with $x^{(i)}_j$ denoting the $j$th query, $y^{(i)}_{j,+}$ denoting the \emph{preferred} response on the query, and $y^{(i)}_{j,-}$ denoting the \emph{not} preferred response on the query. Each dataset may correspond to different tasks such as coding, math, general helpfulness, or safety. 

We focus on preference learning using DPO~\cite{rafailov2023direct}. Let $\pi_\vtheta$ denote a language model parameterized by $\vtheta$ with $\pi_\vtheta(y|x)$ denoting the probability that the model outputs response $y$ given prompt $x$. In DPO, we are given a \emph{reference} model $\pi_\texttt{ref}$ and use it to initialize $\pi_\vtheta(y|x)$. The loss for a particular sample $(x,y_+,y_-)$ is as follows: $\sL( \vtheta ; x, y_+, y_- ) :=$
\begin{align*}
     - \log\left( \sigma \left( \beta \frac{\pi_\vtheta(y_+ | x)}{\pi_\texttt{ref}(y_+ | x)} - \beta \frac{\pi_\vtheta(y_- | x)}{\pi_\texttt{ref}(y_- | x)} \right)\right)
\end{align*}
where $\sigma$ is the sigmoid function and $\beta > 0$ is a hyperparameter.  

To train a model via DPO across $k$ task datasets, the standard practice optimizes a reweighted loss over all datasets \citep{ivison2024UnpackingDA}:
\begin{equation}
    \min_\vtheta \sum_{i=1}^k \alpha_i \sum_{(x, y_+, y_-) \in \sD_i} \sL(\vtheta ; x, y_+, y_-)
    \label{eq:standard}
\end{equation}
where $\alpha_1, \ldots, \alpha_k \in \mathbb{R}$ are positive weights. 
In practice, two choices of dataset weightings are typically used: (1) \text{uniform weighting},  $\alpha_i = 1$ for all $i \in  [k]$, and (2) \text{task-normalized weighting}, $\alpha_i = \frac{1}{|\sD_i|}$ for all $i \in [k]$. 

Both weighting choices have several limitations. Uniform weighting is sensitive to the size of task datasets; if $|\sD_i| > |\sD_j|$, then task $i$ receives more weight since it contributes more to the total loss than task $j$. 
As a result, practitioners often need to run many experiments to determine a dataset mixture that yields strong performance across all tasks \citep{ivison2024UnpackingDA}. 
Task-normalized weighting ensures all tasks are weighted equally by normalizing by the dataset size. However, this rigidity can be suboptimal; if some tasks are easier to learn than others, giving all tasks equal weight may cause the model to focus too much on easy tasks at the expense of harder ones.

To address these challenges with standard data mixing approaches in post-training, in the next section, we introduce AutoMixAlign (AMA), our new framework for balancing performance across tasks. 
To simplify our notation, we henceforth let $z = (x,y_+, y_-)$ denote an example and write $\mathcal{L}(\vtheta, z) :=\sL(\vtheta ; x, y_+, y_-)$. 

\section{AutoMixAlign (AMA)}

In this section, we present AMA. 
We derive the AMA optimization problem, describe two algorithms to optimize it, and then provide theoretical analysis.

\subsection{Deriving the AMA Optimization Problem}
Suppose we have model parameters $\vtheta_1, \dots, \vtheta_k$ where each $\vtheta_i$ was obtained via DPO training on $\sD_i$.
We refer to $\vtheta_i$ as the \textit{specialist} model parameters for task $i$ since $\vtheta_i$ should perform well on task $i$ but may perform poorly on the remaining tasks.\footnote{We train specialist models on individual datasets to ensure they achieve strong performance, though alternative approaches may exist. For instance, it may be possible to train a strong specialist for two tasks $i$ and $j$ by training on $\sD_i \cup \sD_j$.}
%
Since task losses tend to be predictive of performance~\citep{chen2024scaling}, a natural approach would be to find model parameters $\vtheta$ that match the losses achieved by $\vtheta_i$ in each task $i$ by solving the following feasibility problem:
\begin{equation}
\begin{split}
    &\text{Find } \vtheta \text{ s.t. } \quad \forall i\in [k], \\
    &\frac{1}{|\sD_i|}\left[ \sum_{z \in \sD_i} \sL(\vtheta, z) -
    \sum_{z \in \sD_i} \sL(\vtheta_i, z) \right] \leq 0.
\end{split}
\label{eq:feasibility}
\end{equation}
LLMs are typically sufficiently expressive so that there exists a solution to feasibility problem~\ref{eq:feasibility}. However, since this problem is not computationally tractable, we relax it into a minimization problem over constraint violations:
\begin{equation}
\label{eq:minimax_unrelaxed}
\begin{split}
    &\min_\vtheta \max_{i \in [k]} \max 
    \left\{\vphantom{\left[ \sum_{z \in \sD_i} \sL(\vtheta, z) - \sum_{z \in \sD_i} \sL(\vtheta_i, z)  \right]}
    \right. \\
    &\frac{1}{|\sD_i|} \left[ \sum_{z \in \sD_i} \sL(\vtheta, z) - \sum_{z \in \sD_i} \sL(\vtheta_i, z)  \right], 0 
    \left.\vphantom{\left[ \sum_{z \in \sD_i} \sL(\vtheta, z) - \sum_{z \in \sD_i} \sL(\vtheta_i, z)  \right]}\right\}.
\end{split}
\end{equation}
%
%
Unfortunately, we cannot apply stochastic gradient descent directly to optimization problem~\ref{eq:minimax_unrelaxed} because the outer $\max\{\cdot, 0\}$ function makes the gradient computed over a random mini-batch a biased estimator of the true gradient with respect to the full data distribution. To make this optimization problem~\ref{eq:minimax_unrelaxed} tractable, we can upper bound it with the following optimization problem:
\begin{equation}
\label{eq:ama}
\begin{split}
    \min_\vtheta \max_{i \in [k]} \frac{1}{|\sD_i|} \sum_{z \in \sD_i}\max \left\{  \sL(\vtheta, z) - \sL(\vtheta_i, z), 0 \right\}.
\end{split}
\end{equation}
Optimization problem~\ref{eq:ama} is the core optimization problem that AMA aims to solve. In this formulation, $\sL(\vtheta, z) - \sL(\vtheta_i, z)$ denotes the excess loss for sample $z$ on task $i$, and $\sE(\vtheta, \vtheta_i, z) \coloneqq \max \left\{\sL(\vtheta, z) - \sL(\vtheta_i, z), 0 \right\}$ is the \textit{clipped} excess loss, quantifying how much our generalist model’s loss deviates from the loss achieved by our specialist model $\vtheta_i$. 
The excess loss has been used in prior works to prevent optimization from focusing on tasks that have already been learned (\textit{i.e.}, tasks with zero excess loss) at the expense of more challenging tasks~\cite{he2024robust, xie2024doremi}.
For completeness, we further discuss the excess loss in Appendix~\ref{app:excess_loss}.\footnote{
Since specialist model losses $\sL(\vtheta_i, z)$ are fixed throughout training, we reduce AMA's memory footprint by pre-computing  $\sL(\vtheta_i, z)$ and adding them to a new column in each $\sD_i$ before training our generalist model. }

How do we interpret optimization problem~\ref{eq:ama}? By minimizing the maximum clipped excess loss, this optimization aims to train a generalist model that matches the performance of all specialist models $\vtheta_i$. Toward this aim, it focuses on the task with the largest average clipped excess losses. In particular, once it has learned a task so that all the clipped excess losses are zero, it stops optimizing that task; if $\vtheta_i$ maximizes performance on task $i$, then optimizing beyond $\sL(\vtheta_i, z)$ has no benefit and may eventually lead to overfitting.
In the following sections, we discuss two different algorithms to solve this optimization problem: objective reweighting and task resampling.


%

\begin{algorithm*}[h]
\caption{AutoMixAlign Resampling (AMA-S)}
\label{alg:ama}
\begin{algorithmic}[1]
\STATE \textbf{Inputs:} Task datasets $\sD_1, \dots, \sD_k$, specialist model parameters $\vtheta_1, \dots, \vtheta_k$, batch size $b$, smoothing parameter $c\in(0,1)$, training steps $T$
\STATE \textbf{Outputs:} Trained model parameters $\vtheta$
\STATE For all tasks, compute $\sL(\vtheta_i, z)$ for each $z \in \sD_i$, and add them to a new column in $\sD_i$ \label{line:precompute}
\STATE Initialize model parameters $\vtheta$
\STATE $\bm{q}^{(1)} \gets (\frac{1}{k},\dots,\frac{1}{k})$
\FOR{$t = 1, \dots, T$}
    \STATE $\alpha^{(t)}_i \gets (1-c)q^{(t)}_i + c\frac{1}{k}, \forall i$
    \STATE $\sT \sim \text{Multinomial}(\alpha^{(t)}_1, \dots, 
    \alpha^{(t)}_k, b)$
    \STATE $\sB \gets \{z \sim \sD_j : \forall j \in \sT\}$
    \STATE $\sB_i \gets \sB \cap \sD_i$ for each task $i$ 
    \STATE Update $\bm{q}^{(t)}$ using  EXP3 ~\cite{auer2002finite}:
    $$
    \sE(\vtheta, \vtheta_i, z) \coloneqq \max\{ \mathcal L(\vtheta, z)-\mathcal L(\vtheta_i, z), 0\}
    $$
    $$
    q^{(t+1)}_i \gets q^{(t)}_i \exp\left\{ \frac{1}{q^{(t)}_i}\frac{1}{|\sB_i|}\sum_{z\in\sB_i} \sE(\vtheta, \vtheta_i, z)\right\}
    $$
    $$
    q^{(t+1)}_i \gets \frac{q^{(t+1)}_i}{\sum_{j=1}^k q^{(t+1)}_j} 
    $$
    \STATE Update $\vtheta$ using one step of gradient descent on the loss 
    $$
    \frac{1}{|\sB|} \sum_{z\in\sB} \mathcal \max\{\sL(\vtheta, z)-\mathcal \sL(\vtheta_i, z), 0\}
    $$
    
\ENDFOR
\end{algorithmic}
\end{algorithm*}

\subsection{Reweighting Algorithm (AMA-R)}


The optimization problem~\ref{eq:ama} is equivalent to 
\begin{equation}
\label{eq:ama_r}
\begin{split}
    \min_\vtheta \max_{\valpha \in \Delta^k} \sum_{i=1}^k \alpha_i \frac{1}{|\sD_i|} \sum_{z \in \sD_i} \sE(\vtheta, \vtheta_i, z)
\end{split}
\end{equation}
where $\Delta^k = \{\valpha \in \mathbb{R}^k : \sum_{i} \alpha_i = 1, \alpha_i \geq 0 \}$ is a probability simplex. 
To optimize this new objective, we follow \citet{sagawa2019distributionally} and interleave gradient updates on task weights $\valpha$ and model parameters $\vtheta$, increasing the weight on tasks with larger excess losses.
Specifically, we sample a batch of data from each task uniformly at random,
update $\valpha$ using exponentiated gradient ascent on the clipped excess loss of each task, and then update $\vtheta$ using the loss $\sum_{i=1}^k \alpha_i \frac{1}{|\sD_i|} \sum_{z \in \sD_i} \sE(\vtheta, \vtheta_i, z)$ and a standard first-order optimization algorithm. Since this algorithm reweights the objective using $\alpha$, we call it AMA Reweighting (AMA-R).
Due to space constraints, we present this algorithm in Algorithm~\ref{alg:ama_r} of Appendix~\ref{app:algorithms}.
For this interleaved optimization approach, \citet{sagawa2019distributionally} previously established a convergence rate of $O({1}/{\sqrt{T}})$ in the convex setting.\footnote{While the convergence rate of AMA-R follows from  \citet{sagawa2019distributionally}, we note that AMA-R optimizes a different objective; \citet{sagawa2019distributionally} optimizes the loss $\sL(\vtheta, z)$ whereas AMA-R optimizes the clipped excess loss $\sE(\vtheta, \vtheta_i, z)$.}
%

\subsection{Resampling Algorithm (AMA-S)}

The reweighting algorithm described in the previous section prioritizes tasks by adjusting the objective weights $\valpha$. However, when some weights $\alpha_i$ are very small, this method is inefficient; we incur the same cost to compute each task's contribution to the gradient, but tasks with small weights have little impact on the model update.
To illustrate, suppose the extreme: all weight is concentrated on task $i$ ($\alpha_i = 1, \alpha_j = 0 \; \forall j \neq i$). Letting $b$ denote the batch size,
AMA-R samples $b/k$ examples from every task, but examples from all tasks other than task $i$ are wasteful as they don't contribute to the gradient calculation. 

We can alternatively use $\valpha$ to control the probability of sampling from each task's dataset when optimizing problem~\ref{eq:ama}. We refer to this resampling algorithm as AMA Resampling (AMA-S) and provide its implementation in Algorithm~\ref{alg:ama}. 
Like AMA-R, AMA-S alternates between updating $\valpha$ and $\vtheta$. For a given $\valpha$, we sample batches according to $\text{Multinomial}(\alpha_1, \dots, \alpha_k)$ and update $\vtheta$ using standard first-order optimization. The $\valpha$ updates use the EXP3 online learning algorithm~\citep{auer2002finite} to adaptively adjust sampling probabilities to focus computation on tasks with large excess loss. 
For EXP3 updates, we maintain an internal weight vector $\bm{q}\in\Delta^k$ and compute the sampling distribution as $\valpha = (1-c)\bm{q}+c\frac{1}{k}\bm{1}$, where $c\in(0,1)$ smooths the sampling distribution towards a uniform distribution.
With AMA-S, all examples from every task have equal weight in the gradient calculation, improving the computational efficiency when task weights are unequal.
In Section~\ref{app:reweighting_vs_resampling}, we empirically validate this argument.


Since the convergence proof given by \citet{sagawa2019distributionally} applies only to the reweighting algorithm, we provide a new convergence result for this resampling based algorithm.
In the next section, we will show that AMA-S can be viewed as a two-player zero-sum game between an $\valpha$-player and a $\vtheta$-player and prove a standard $O(1/\sqrt{T})$-style convergence rate in the convex case.

\subsection{Theoretical Analysis}
\label{sec:theory}

In this section, we show that AMA-S is a theoretically principled multi-task learning algorithm. Again, we consider the setting where we have $k$ datasets $\sD_1, \ldots, \sD_k$, each with $\sD_i = \{x^{(i)}_1, \ldots, x^{(i)}_{n}\}$. We consider the function class $\{ \vtheta \in \Theta \}$ and suppose $\sD_i \subset \mathcal{Z} $ for all $i \in [k]$. For simplicity, we abstract away from the excesses losses and consider the following:
\begin{align*}
        \min_{\vtheta \in \Theta} \max_{\valpha \in \Delta^k} \sum_i \alpha_i \frac{1}{|\sD_i|} \sum_{z \in \sD_i} \sL( \vtheta, z).
\end{align*}
We can model this problem as a zero-sum game between two players, the $\vtheta$-player and the $\valpha$-player. The protocol is as follows. For round $t=1,2,\ldots$
\begin{enumerate}
    \item The  $\valpha$-player chooses $\bm{\alpha}_t \in \Delta^k$.
    \item The $\vtheta$-player chooses $\vtheta_t \in \Theta$.
    \item The $\vtheta$-player samples $i_t \sim \bm{\alpha}_t$.
    \item The $\vtheta$-player suffers loss $\frac{1}{|\sD_{i_t}|} \sum_{(x,y) \in \sD_{i_t}} \sL(\vtheta_{t}, x)$.
\end{enumerate}
For our convergence result, we require access to a $\vtheta$-player that satisfies the following property. 
\begin{definition}
    We say that the $\vtheta$-player is $C$-low regret if for any sequence of indices $\{i_1, \ldots i_T\}$, we have that 
    \begin{align*}
        & \sum_{t=1}^T \frac{1}{|\sD_{i_t}|} \sum_{z \in \sD_{i_t}} \mathcal{L}(\vtheta_t, z) & \\
        & \qquad \leq  \min_{\vtheta \in \Theta} \frac{1}{|\sD_{i_t}|} \sum_{z \in \sD_{i_t}} \mathcal{L}(\vtheta, z) + C\sqrt{T}.
    \end{align*}
\end{definition}
\noindent Now, we show a case where such a $C$-low-regret $\vtheta$-player exists. Suppose that $z = (x,y)$, $f_\vtheta(x) = \vtheta \cdot x$, $\sL(\vtheta, x) = l(f_\vtheta(x), y)$. Note that we have $f_\vtheta(x) = w \cdot \phi(x)$ where $\phi$ is a feature extractor. Then, online gradient descent is an example of such a low-regret $\vtheta$-player. 

\begin{example}\label{ex:ogd}
    Let $l(f_\vtheta(x), y) = (\vtheta \cdot x - y)^2$, $\Theta = \{\vtheta : \norm{\vtheta}_2 \leq 1 \}$, $y \in [-1,1]$, and $x \in \{x : \norm{x}_2 \leq 1 \}$. Then, online gradient descent is $C$-low-regret algorithm with $C = 4\sqrt{2}  $. 
\end{example}

\noindent See the Appendix for a proof based on online convex optimization \cite{shalev2012online}.

Now, we can prove our main theoretical result under mild simplifying assumptions. It shows that by framing AMA-S as a zero-sum two-player game, the $\vtheta$-player minimizes task losses while the $\valpha$-player adversarially upweights challenging tasks, leading to balanced convergence.

\begin{theorem}\label{thm:main_result}
    Fix $\epsilon, \delta \in (0,1)$. Suppose that the $\vtheta$-player is a $C$-low-regret algorithm.  Suppose $\max_{\vtheta \in \Theta} \sL(\vtheta, z) \leq 1$ for all $z \in \mathcal{Z}$. Then, if $T \geq \max(\frac{64 C^2}{\epsilon^2}, \frac{32 k \log(k)}{\epsilon})$ and $\eta = \frac{1}{4k}$, we have that
    \begin{align*}
        &  \max_{i \in [k]} \frac{1}{T}\sum_{t=1}^T \frac{1}{|\sD_i|} \sum_{z \in \sD_i} \sL(\vtheta_t, z) \\
        & \qquad \leq \min_{\vtheta \in \Theta} \max_{i \in [k]} \frac{1}{|\sD_i|} \sum_{z \in \sD_i} \sL(\vtheta, z) + \epsilon.
    \end{align*}
\end{theorem}

\noindent This result establishes a standard $O({1}/{\sqrt{T}})$ convergence rate for the AMA-S algorithm on multi-task learning under mild assumptions. The argument extends classical arguments from \cite{shalev2016minimizing} to the multi-task setting and beyond the realizable setting to the agnostic setting where there may not be a single predictor that perfectly fits the data.

\section{Experiments}
\label{sec:experiments}

 
We consider the following baselines:
\begin{enumerate}
    \item \textbf{Standard optimization (Standard)}: Combine all task datasets and minimize the total loss---the standard approach in DPO training. 
    \item \textbf{Standard optimization with uniform task sampling (Standard Uniform)}: Minimize the total loss across tasks but sample data from each task with probability $1/k$. 
    \item \textbf{Model Averaging}~\citep{yang2024model}: A heuristic model merging method that averages the parameters of the specialist models, giving equal weight to each model.
\end{enumerate}

\begin{table*}[h]
    \centering
    \resizebox{\textwidth}{!}{
    \begin{tabular}{l|c|c|c|c|c}
    \hline
    \rowcolor[HTML]{EFEFEF} 
    \textbf{Model} & \textbf{IFEval (\%)} $\uparrow$ & \textbf{Alpaca Eval (\%)} $\uparrow$ & \textbf{MBPP (\%)} $\uparrow$ & \textbf{HumanEval (\%)}$\uparrow$ & \textbf{Average (\%)} $\uparrow$\\ \hline
    Zephyr-7b-sft-full & 35.30 &  8.64 & 49.90 & 50.46 & 36.08 \\ \hline
    Chatbot Arena 2024 Specialist & 43.44 & 16.05 & 44.60 & 50.61 & 38.68 \\ 
    CodeUltraFeedback  Specialist & 35.86 &  9.93 & 52.16 & 57.32 & 38.82 \\ \hline\hline
    Standard              & 39.74 & 15.23 & 37.77 & 47.82 & 35.14 \\
    Standard Uniform      & 42.51 & 11.20 & 36.33 & 51.83 & 35.47 \\ 
    Model Averaging    & 42.14 & 11.51 & 48.56 & \textbf{55.49} & 39.43 \\ 
    \hline
    AMA Reweighting        & 42.51 & \textbf{18.15} & \textbf{51.44} & 54.88 & \textbf{41.75} \\
    AMA Resampling        & \textbf{43.44} & 15.61 & 51.08 & 50.61 & 40.19 \\ \hline
    \end{tabular}
    }
    \caption{Setup 1, Helpfulness (Chatbot Arena 2024) + Coding (CodeUltraFeedback).}
    \label{tab:chatbot_codeuf}
\end{table*}

\begin{figure}
    \centering
    \includegraphics[width=\linewidth]{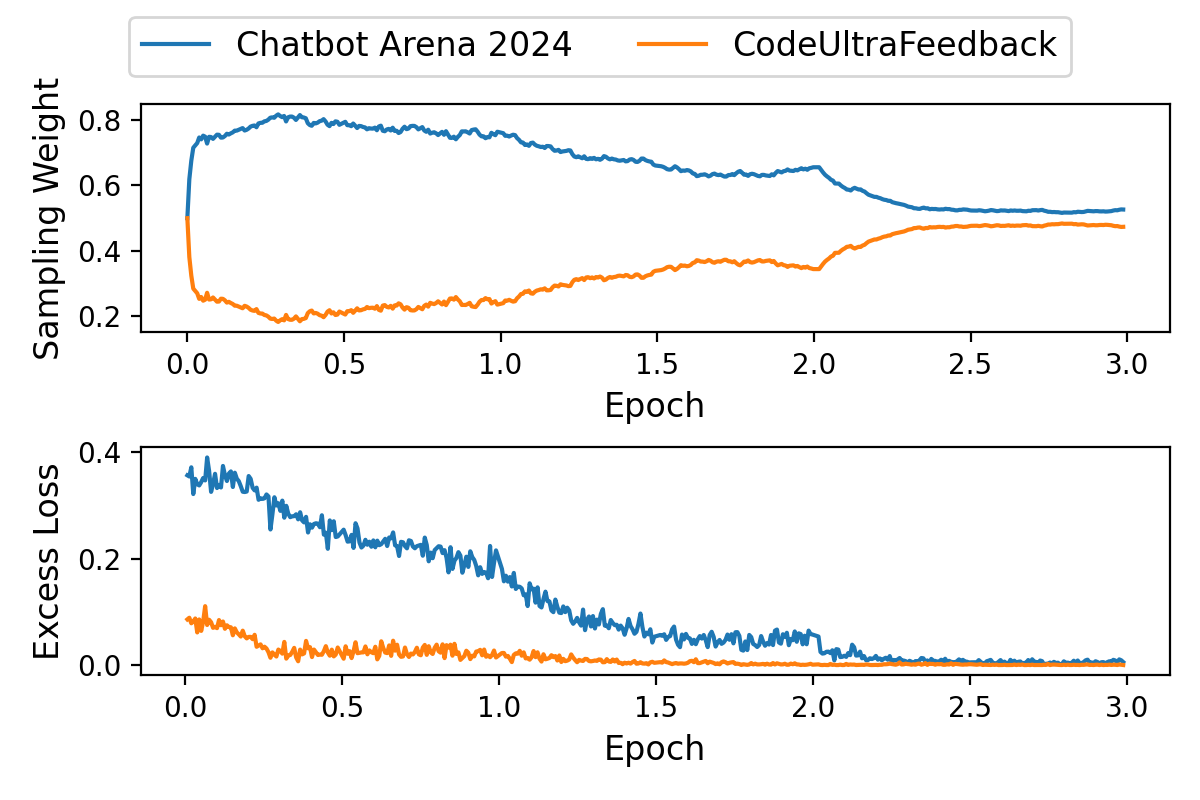}
    \caption{Excess losses and task weights for AMA-S in Setup 1, Helpfulness (Chatbot Arena 2024) + Coding (CodeUltraFeedback).}
    \label{fig:chatbot_codeuf_metrics_ama_s}
\end{figure}

%
%
We provide pseudocode for Standard and Standard Uniform in Algorithm~\ref{alg:standard} and~\ref{alg:standard_uniform} of Appendix~\ref{app:algorithms}.

Our experiments focus on
Helpfulness (Alpaca Eval 2.0~\citep{dubois2024length} and IFEval ~\citep{zhou2023instruction}), Harmlessness (Toxigen~\citep{hartvigsen2022toxigen}), and Coding (MBPP~\citep{austin2021program}, HumanEval~\citep{chen2021evaluating}). 
We consider the following datasets: UltraFeedback ~\citep{cui2023ultrafeedback}, Chatbot Arena 2024~\citep{zheng2023judging}, CodeUltraFeedback~\citep{weyssow2024codeultrafeedback}, and PKU-SafeRLHF with harmlessness
preference labels (SafeRLHF)~\citep{beavertails}.
Each dataset primarily improves performance on only one of three tasks, so we combine datasets to investigate how to produce a model that performs well across all tasks. 
We consider the following dataset combinations: Chatbot Arena 2024 + CodeUltraFeedback (Helpfulness + Coding), UltraFeedback + SafeRLHF (Helpfulness + Harmlessness), and Chatbot Arena 2024 + Coding + SafeRLHF (Helpfulness + Coding + Harmlessness).

In all experiments, we use the open source Zephyr 7b SFT Full~\citep{tunstall2023zephyr} as the base model, and train all models using DPO~\cite{rafailov2023direct}.
To obtain specialist models, we train separate models via DPO on each task dataset for 3 epochs.
We provide full training details for both specialist and generalist models Appendix~\ref{app:training} as well as evaluation details in Appendix~\ref{app:eval}. 
%
%
Due to space constraints, we provide ablations on different features of AMA and different model merging approaches in Appendix~\ref{app:ablations}.

\begin{table*}[h]
    \centering
    \begin{tabular}{l|c|c|c|c}
    \hline
    \rowcolor[HTML]{EFEFEF}  
    \textbf{Model}         & \textbf{IFEval (\%)}   $\uparrow$ & \textbf{Alpaca Eval (\%)}   $\uparrow$ & \textbf{Toxigen (\%)}   $\uparrow$ & \textbf{Average (\%)}   $\uparrow$\\ \hline
    Zephyr-7b-sft-full     & 35.30 & 8.64  & 51.60  & 31.85  \\ \hline
    UltraFeedback Specialist          & 42.51 & 17.59 & 36.10  & 32.07  \\ 
    SafeRLHF Specialist        & 29.02 & 6.66  & 100.00 & 45.23  \\ \hline\hline
    Standard                   & 41.22 & 6.90  & \textbf{99.99}  & 49.37  \\ 
    Standard Uniform           & 44.18 & 15.51 & 99.93  & 53.21  \\ 
    Model Averaging         & 39.37 & 12.21 & 99.96  & 50.51  \\ 
    \hline
    AMA Reweighting         & \textbf{44.55} &  16.21 & 99.74  & 53.50  \\
    AMA Resampling              & \textbf{44.55} & \textbf{16.88} & \textbf{99.99}  & \textbf{53.81}  \\ \hline
    \end{tabular}
    \caption{Setup 2, Helpfulness (UltraFeedback) + Harmlessness (SafeRLHF).}
    \label{tab:uf_harm}
\end{table*}

\begin{table*}[h]
    \centering
    \resizebox{\textwidth}{!}{
    \begin{tabular}{l|c|c|c|c|c|c}
    \hline
    \rowcolor[HTML]{EFEFEF} 
    \textbf{Model} & \textbf{IFEval (\%)}   $\uparrow$ & \textbf{Alpaca Eval (\%)}   $\uparrow$ & \textbf{MBPP (\%)}   $\uparrow$ & \textbf{HumanEval (\%)}   $\uparrow$ & \textbf{Toxigen (\%)}   $\uparrow$ & \textbf{Average (\%)}   $\uparrow$\\ \hline
    Zephyr-7b-sft-full & 35.30 & 8.64  & 49.90 & 50.46 & 51.60  & 39.18 \\ \hline
    Chatbot Arena 2024 Specialist & 43.44 & 16.05 & 44.60 & 50.61 & 69.77  & 44.89 \\ 
    CodeUltraFeedback Specialist & 35.86 & 9.93  & 52.16 & 57.32 & 76.31  & 46.31 \\ 
    SafeRLHF Specialist & 29.02 & 6.66  & 48.56 & 50.61 & 100.00 & 46.97 \\ \hline\hline
    Standard              & 38.63 & 11.32 & 50.00 & 54.88 & 99.91  & 50.95 \\ 
    Standard Uniform      & 47.50 & 11.78 & 27.70 & 37.80 & \textbf{99.99}  & 44.96 \\ 
    Model Averaging    & 42.70 & 12.28 & 52.16 & 53.66 & 96.70  & 51.50 \\ 
    \hline
    AMA Reweighting         & \textbf{48.43} & \textbf{17.77} & 54.32 & \textbf{55.49} & 95.91  & \textbf{54.38} \\

    AMA Resampling        & 43.62 & 13.04 & \textbf{55.76} & 53.66 & 99.80  & 53.18 \\ \hline
    \end{tabular}
    }
    \caption{Setup 3, Helpfulness (Chatbot Arena 2024) + Coding (CodeUltraFeedback) + Harmlessness (SafeRLHF).}
    \label{tab:chatbot_codeuf_harm}
\end{table*}

\subsection{Setup 1: Helpfulness + Coding}

Our first experiment focuses on balancing helpfulness and harmlessness using Chatbot Arena 2024 and SafeRLHF.
As shown in Table~\ref{tab:chatbot_codeuf}, our Chatbot Arena 2024 specialist  improves helpfulness but not coding.
On the other hand, our CodeUltraFeedback specialist marginally improves helpfulness but greatly improves coding.

From Table~\ref{tab:chatbot_codeuf}, we see that both AMA algorithms achieve the highest average scores and balance performance across both tasks. In fact, they are the only methods that do not degrade performance on at least one benchmark. 
On the other hand, Standard and Standard Uniform degrade coding performance while Model Averaging scores lower than both AMA algorithms on AlpacaEval and MBPP.

To understand how AMA balances tasks,
Figure~\ref{fig:chatbot_codeuf_metrics_ama_s} shows task-specific excess losses and task weights during AMA-S training. 
As expected, AMA-S initially shifts weights to the task with higher excess loss, Chatbot Arena 2024, and then increases the weight on CodeUltraFeedback as the gap between the two tasks decreases.
Since the excess loss on Chatbot Arena 2024 remains slightly larger than that of CodeUltraFeedback at convergence, the weight on Chatbot Arena 2024 also remains slightly larger.
We plot task weights and excess losses for AMA-R in Figure~\ref{fig:chatbot_codeuf_metrics_ama_r} of Appendix~\ref{app:training} and observe qualitatively similar behavior. 

\subsection{Setup 2: Helpfulness + Harmlessness}

Our second experiment focuses on balancing helpfulness and harmlessness using UltraFeedback and SafeRLHF.
As shown in Table~\ref{tab:uf_harm}, our UltraFeedback specialist improves helpfulness, but leaves substantial room for improvement on Toxigen. 
However, our SafeRLHF specialist achieves a perfect score on Toxigen but degrades helpfulness. 
%

In Table~\ref{tab:uf_harm}, both AMA algorithms achieve the highest average scores and achieve the highest scores on each benchmark individually. Here, AMA-S and AMA-R perform similarly.
We observe that Standard optimization can be sensitive to the task weighting:
Standard Uniform performs similarly to both AMA algorithms, but Standard \textit{degrades} performance on AlpacaEval significantly.
Model averaging offers only modest improvements on helpfulness benchmarks.
%

\subsection{Setup 3: Helpfulness + Coding + Harmlessness}

Our third experiment builds off of Setup 1 and focuses on balancing helpfulness, coding, and harmlessness using Chatbot Arena 2024, CodeUltraFeedback, and SafeRLHF.
As shown in Table~\ref{tab:chatbot_codeuf_harm}, the Chatbot Arena 2024 and CodeUltraFeedback specialists have significant room for improvement on Toxigen, while our SafeRLHF specialist achieves a perfect Toxigen score but degrades performance on helpfulness and coding.
%

As shown in Table~\ref{tab:chatbot_codeuf_harm}, both AMA algorithms achieve the highest average score across benchmarks while also improving on each benchmark individually. AMA-R outperforms AMA-S on IFEval, AlpacaEval, and HumanEval.  
We again observe that Standard optimization can be sensitive to task weighting: Standard scores 6 points higher on average than Standard Uniform.
Model Averaging balances task scores but nevertheless achieves relatively low scores.
%

\subsection{Resampling vs. Reweighting}
\label{app:reweighting_vs_resampling}

Although we can optimize the AMA objective using reweighting or resampling, we now show that AMA converges more slowly with certain dataset mixtures when using reweighting.
In particular, we show how \emph{uniformly} sampling across datasets (as done by AMA-R) can result in much slower convergence compared to adaptive sampling (as done by AMA-S). 
We train AMA-S and AMA-R for one epoch on a Coding + Harmlessness dataset mixture consisting of 1 copy of a $10^4$-sample subset of SafeRLHF and 10 copies of a 250-sample subset of CodeUltraFeedback \textit{each representing separate tasks} so that we have $k=11$ tasks.

How might AMA-R struggle in this setting?
Let $b$ denote the batch size so that one epoch corresponds to $\frac{1}{b}(250\cdot 10 + 10^4)$ updates. 
Since AMA-R samples tasks uniformly, it will in expectation observe $\frac{b}{11}\cdot \frac{1}{b}(250\cdot 10 + 10^4) \approx 1136$ SafeRLHF samples in one epoch, which may not be sufficient to produce a harmless model.
In contrast, AMA-S can adaptively increase the probability of sampling from SafeRLHF to improve harmlessness.

In Figure~\ref{fig:reweighting_vs_resampling}, we observe that AMA-S quickly converges to almost $100\%$ non-toxicity, whereas AMA-R reaches a maximum of $90\%$. To understand why, we also plot the probability of sampling from SafeRLHF with AMA-R (which is always $1/11$) and AMA-S. We see that AMA-S quickly places significantly more weight on SafeRLHF and thus observes more SafeRLHF data than AMA-R.



\begin{figure}
\centering
\includegraphics[width=\linewidth]{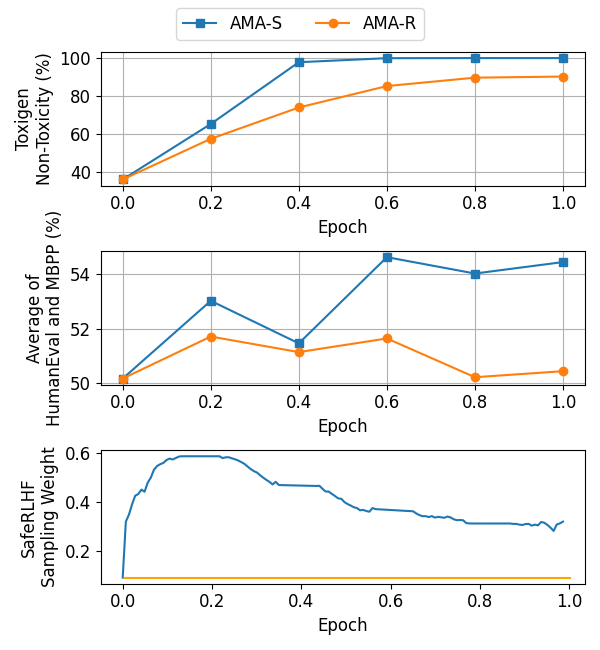}
\caption{AMA-R and AMA-S in Coding (CodeUltraFeedback) + Harmlessness (SafeRLHF).}
\label{fig:reweighting_vs_resampling}
\end{figure}

\section{Related Work}

\noindent \textbf{Data Mixing.} 
Prior works often use heuristic filters~\cite{brown2020language, du2022glam, chowdhery2023palm} or
ablations studies to produce a data mixture that yields strong performance~\cite{ivison2024UnpackingDA}. 
%
%
%
%
Recently, many works have leveraged group distributionally robust optimization (GroupDRO) to automatically determine an appropriate data mixture~\citep{oren2019distributionally, sagawa2019distributionally,  michel2021balancing, albalak2023efficient, xia2023sheared, xie2024doremi, he2024robust}. 
Similar to AMA, DRO methods use a minimax optimization to optimize the worst-case loss or excess loss.
However, AMA has two core differences: (1) AMA focuses on alignment rather than pretraining, and (2) AMA optimizes the clipped excess loss during model updates rather than the loss or excess loss.
We highlight other algorithmic differences between AMA and closely-related DRO-style algorithms in Appendix~\ref{app:ama_vs_doremi}.
\citet{pan2024scalebio} use bilevel minimax optimization to reweight tasks for LLM fine-tuning, though this approach also minimizes the loss rather than the clipped excess loss. \citet{de2024dynamic} consider dynamically weighting multiple dimensions of alignment, but they focus on controllable generation and two-stage RLHF which assume reward models are separately trained before LLM alignment.  \newline 


\noindent \textbf{Gradient Conflict Resolution.} Multi-task learning methods based on multiple-gradient descent (MGDA)~\citep{desideri2009multiple} compute gradients that reduces the loss of all tasks at every step in a balanced fashion~\citep{sener2018multi, navon2022multi, yu2020gradient, liu2021towards, liu2021conflict, fernando2022mitigating}. 
%
%
Other methods seek to eliminate conflicting components of task-specific gradients~\citep{chen2020just, yu2020gradient} or simply rescale task-specific gradients to have similar norms~\citep{chen2018gradnorm}. 
These gradient-based approaches differ fundamentally from data reweighting and resampling methods, as they address multi-task conflicts through geometric operations in gradient space rather than through adjustments to the loss or training distribution. \newline
%
%

\noindent \textbf{Model Merging.} Model merging methods combine the parameters of separate models to produce a new model that shares the capabilities of each individual model.  
%
%
%
The most basic approach is to uniformly average the parameters across models~\citep{mcmahan2017communication, stich2018local, wortsman2022model}, though non-uniform merging methods also exist~\citep{matena2022merging, jin2022dataless, tam2024merging}.
%
%
%
While model merging is cheaper than fine-tuning---the approach taken in this paper---model merging is heuristic in nature and generally leads to worse performance than fine-tuning ~\cite{yang2024model}. 






\section{Conclusion}

We introduced AutoMixAlign (AMA), an adaptive data mixing strategy that uses minimax optimization to prioritize tasks with larger excess loss. 
We provide two algorithms for AMA: one adaptively reweights how much each task contributes to the objective function (AMA-R), while the other adaptively updates how much data is sampled from each task during model updates (AMA-S).
Both algorithms are theoretically justified and have $O(1/\sqrt{T})$ convergence rates in the convex case. AMA-R's convergence result follows from \citet{sagawa2019distributionally}, and we provide a proof of AMA-S's convergence rate.
We evaluate both AMA algorithms on several multitask alignment setups, and observe that they outperform the standard alignment approach which simply optimizes the total loss across all tasks and also outperform model merging methods.
We outline directions for future exploration in Appendix~\ref{app:future_directions}.


\section{Limitations}
\label{sec:limitations}
\noindent \textbf{Specialist model training overhead.}
AMA computes excess losses using separate specialist models trained on each task dataset individually.
Training these specialist models can introduce significant  training overhead compared to other multi-task learning methods like MGDA~\citep{desideri2009multiple}, model merging~\citep{wortsman2022model}, or simply optimizing the total loss, especially when the number of tasks is large or when the size of each task dataset is large. 
However, we emphasize that AMA requires the compute equivalent to \textit{just two optimization runs of DPO.} 
More concretely, training separate specialist models for $m$ epochs requires $\sum_{i=1}m|\sD_i| = m|\sD|$ model updates, and training the generalist model requires an additional $m|\sD|$ model updates---a total of $2m|\sD|$ model updates.
In other words, the ratio of the number of training steps needed for AMA training vs. standard DPO does not depend on the number of tasks. 
Since the common practice in the literature is to run many data mixture ablations to optimize performance across many tasks (\textit{e.g.}, \citet{ivison2024UnpackingDA}), AMA in its current form can already improve the efficiency of the practice.
%
Existing techniques in the literature can reduce the expense of computing specialist losses. For instance, 
ExcessMTL~\citep{he2024robust} uses a diagonal estimation of the generalist model's Hessian matrix to approximate excess losses, while 
DoReMi and RegMix~\citep{liu2024regmix} estimate specialist losses using models that are much smaller than the generalist model.\footnote{In these works, the specialist models are called \textit{reference models}. We use the term ``specialist'' throughout this work to avoid confusion with the reference model $\pi_\text{ref}$ in the DPO loss.} 
While DoReMi~\cite{xie2024doremi} trains a single multitask model to compute excess losses, we note that this approach requires just as many model updates as training $k$ specialist models.
AMA has the added advantage that specialist models can be trained in parallel, reducing the total time required to train them by a factor of $1/k$.\newline

\noindent \textbf{AMA is specific to DPO.} Our experiments and theoretical contributions focus on preference alignment with DPO, though empirical findings may not transfer to other alignment techniques like on-policy learning or supervised fine-tuning. However, one could in principle borrow the algorithmic ideas discussed in this work and use reweighting and resampling to prioritize tasks in other parts of the learning stack, such as on-policy learning or supervised fine-tuning. 

\section*{Impact Statement}

This paper focuses on aligning LLMs to multiple tasks while achieving strong performance across each task individually and represents a significant step towards the development of generalist LLMs.
We showcase how our method encourages LLMs to balance its capabilities (\textit{e.g}, to be helpful while also being harmless), and  we expect that our findings will have a positive impact on the research community as well as industry practitioners.
While it is possible that researchers or practitioners might use our proposed method for negative purposes (\textit{e.g.} training an LLM to balance multiple toxic capabilities), we maintain that this work is largely positive, and that potential negative societal impacts of this work are minimal.
%






\bibliography{main}

\newpage
\appendix
\onecolumn

\addcontentsline{toc}{section}{Appendix} 
\part{Appendix} 
\parttoc 

\clearpage
\section{Theory}

In this section, we prove the theoretical results from Section \ref{sec:theory}. We begin with some definitions. A sequence \(B_1, \ldots, B_T\) of random variables is said to be Markovian if, for any \(t \in [T]\), the variable \(B_t\) is independent of \(B_1, \ldots, B_{t-2}\) given \(B_{t-1}\). Additionally, a sequence \(A_1, \ldots, A_T\) of random variables forms a martingale difference sequence relative to \(B_1, \ldots, B_T\) if the conditional expectation \(E[A_t | B_1, \ldots, B_t]=0\) for every \(t \in [T]\).

\subsection{Lemmas}

The Exponential Gradient (EG) algorithm generates a series of vectors, denoted as \(z_1, \ldots, z_T\), each belonging to \(\mathbb{R}^k\). Let \(\eta > 0\). Initialize \(\tilde{q}_1 = (1, \ldots, 1) \in \mathbb{R}^k\). For \(t \geq 1\), 
\begin{itemize}
    \item Set \(\tilde{q}_{t+1, i} = \tilde{q}_{t,i} \exp(-\eta z_{t,i})\) for each \(i \in [k]\)
    \item Normalize \(\tilde{q}_t\) to obtain \(q_t\) such that \(q_t = \frac{\tilde{q}_t}{\sum_{i=1}^k \tilde{q}_{t,i}}\)
\end{itemize}
The EG algorithm satisfies the following Lemma.
\begin{lemma}[\citet{shalev2012online}, Theorem 2.22]\label{lem:eg}
    Assume that $\eta z_{t,i} \geq -1$ for every $t$ and $i$. Then, for every $u \in \Delta^k$, we have that
    \begin{align*}
        \sum_{t=1}^T (q_t - u)^\top z_t \leq \frac{\log(k)}{\eta} + \eta \sum_{t=1}^T \sum_{i=1}^k q_{t,i} z_{t,i}^2.
    \end{align*}
\end{lemma}

\begin{lemma}[\citet{hazan2011beating, fan2012hoeffding}]\label{lem:martingale}
    Let $B_1, \ldots, B_T$ be a Markovian sequence and let $A_1, \ldots, A_T$ be a martingale difference sequence wrt $B_1, \ldots, B_T$. Assume that for every $t \in [T]$, we have $|A_t| \leq V$ and $\be[A_t^2 | B_1, \ldots, B_T] \leq s$. Then, for every $\alpha > 0$, we have that
    \begin{align*}
        \bp( \frac{1}{T} \sum_{t=1}^T A_t \geq \alpha) \leq \exp(- T\frac{\alpha^2}{2s + 2\alpha V/3}).
    \end{align*}
    In particular, for every $\delta \in (0,1)$, if 
    \begin{align*}
        T \geq \frac{2(s + \alpha V / 3) \log(1/\delta)}{\alpha^2}
    \end{align*}
    then with probability of at least $1-\delta$ we have that $\frac{1}{T} \sum_{t=1}^T A_t \leq \alpha$.
\end{lemma}

\subsection{Proof of Theorem \ref{thm:main_result}}

\begin{proof}[Proof of Theorem \ref{thm:main_result}]

To simplify notation, moving forward, we define:
\begin{align*}
    L_i(\vtheta) := \frac{1}{|\sD_{i}|} \sum_{x \in \sD_{i}} \sL(\vtheta, z).
\end{align*}
\textbf{Step 1: Concentration of Probability.} To begin, we note that by definition \( \alpha_{t,i} = \frac{1}{2k} + \frac{q_{t,i}}{2} \), implying that for all $i \in [k]$ and $t \in [T]$,
\begin{align*}
 \frac{q_{t,i}}{\alpha_{t,i}} = \frac{q_{t,i}}{\frac{1}{2k} + \frac{q_{t,i}}{2}} \leq 2   
\end{align*}
and
\begin{align*}
    \frac{1}{\alpha_{t,i}} \leq 2k.
\end{align*}
Fix $i \in [k]$ and let \( A_1, \ldots, A_T \) be a martingale difference sequence where \( A_t = L_i(\vtheta_t) - \langle e_i, \frac{L_{i_t}(\vtheta_t)}{\alpha_{t,{i_t}}} e_{i_t} \rangle \). We will apply Lemma \ref{lem:martingale}. Note that \( |A_t| \leq (2k + 1) \). Furthermore, \( \E[A_t^2 | q_t, \vtheta_t] \leq 2k \) since
\[
\E\left[ \langle e_{i_t} \frac{L_{i_t}(\vtheta_t)}{\alpha_{t,i_t}}, e_i \rangle^2 | q_t, \vtheta_t \right] = \sum_{j=1}^k \frac{\alpha_{t,j}}{\alpha_{t,j}^2} \langle e_j L_j(\vtheta_t), e_i \rangle^2 \leq \frac{1}{\alpha_{t,i}} \leq 2k.
\]
where in the first inequality we used $L_j(\theta_t) \leq 1$ by assumption and in the second inequality we used $\frac{1}{\alpha_{t,i}} \leq 2k$.
  Then by Lemma \ref{lem:martingale} if \( T \geq 6k \log(k/\delta) / (\varepsilon/8)^2 \) we have that with probability at least \( 1 - \delta/k \), \( \frac{1}{T} \sum_{t=1}^T A_t \leq \varepsilon/8 \). Then, by the union bound, we have that with probability at least $1-\delta$
\begin{align}
   \forall i \in [m], \quad \frac{1}{T} \sum_{t=1}^T L_i(\vtheta_t) \leq \frac{1}{T} \sum_{t=1}^T \langle e_i, \frac{L_{i_t}(\vtheta_t)}{\alpha_{t,i}} e_{i_t} \rangle + \varepsilon/8. \label{eq:concentration_probability}
\end{align}
For the rest of the proof, we condition on the above event. \newline

\noindent \textbf{Step 2: the $\valpha$-player.} 
Setting \( z_t = -\frac{L_{i_t}(\vtheta_t)}{\alpha_{t, i_t}} e_{i_t} \), we see that the \( \valpha \)-player applies the EG algorithm wrt \( z_1, \ldots, z_T \). Since $z_{t,i} \geq -2k$ and $\eta \leq \frac{1}{2k}$, Lemma \ref{lem:eg} implies that for every $u \in \Delta^k$,
\begin{align}
    \frac{1}{T} \sum_{t=1}^T \langle q_t - u, z_t \rangle & \leq \frac{\log(k)}{\eta T} + \frac{\eta}{T} \sum_{t=1}^T \sum_{i=1}^k q_{t,i} z_{t,i}^2 \nonumber \\
    & \leq \frac{\log(k)}{\eta T} + \frac{4k\eta}{T} \sum_{t=1}^T  L_{i_t}(\vtheta_t)^2 \nonumber \\
    & \leq \frac{\log(k)}{\eta T} + \frac{4k \eta }{T} \sum_{t=1}^T L_{i_t}(\vtheta_t) \nonumber \\
    & \leq \varepsilon / 8 + \frac{1}{T} \sum_{t=1}^T L_{i_t}(\vtheta_t) \label{eq:alpha_player}
\end{align}
where in the last inequality we used \( \eta = \frac{1}{4k} \) and \( T \geq \frac{32 k \log(k)}{\epsilon} \). \newline

\noindent \textbf{Step 3: the $\vtheta$-player.} By definition of our low regret $\vtheta$-player, we have that for every sequence $i_1, \ldots, i_T$
\begin{align*}
     \frac{1}{T} \sum_{t=1}^T  L_{i_t}(\vtheta_t) - \min_{\vtheta \in \Theta} \frac{1}{T} \sum_{t=1}^T L_{i_t}(\vtheta) & \leq \frac{C }{\sqrt{T}} \\
     & \leq \frac{\epsilon}{8}
\end{align*}
where we used $T \geq \frac{64 C^2}{\epsilon^2}$. Then, rearranging and using the fact that the average is always less than or equal to the max, we have:
\begin{align}
    \frac{1}{T} \sum_{t=1}^T  L_{i_t}(\vtheta_t) & \leq \min_{\vtheta \in \Theta} \frac{1}{T} \sum_{t=1}^T L_{i_t}(\vtheta) +  \frac{\epsilon}{8} \nonumber \\
    & \leq \min_{\vtheta \in \Theta} \max_{i \in [k]} L_{i}(\vtheta) + \frac{\epsilon}{8}. \label{eq:w_player}
\end{align}

\noindent\textbf{Step 4: Wrapping it up.} Putting it all together, we have that for all $i \in [k]$,
\begin{align*}
    \frac{1}{T} \sum_{t=1}^T L_i(\vtheta_t) & \leq \frac{1}{T} \sum_{t=1}^T \langle e_i, \frac{1}{\alpha_{t,i}} L_i(\vtheta_t) e_i \rangle + \varepsilon/8 \\
    & \leq \frac{1}{T} \sum_{t=1}^T L_{i_t}(\vtheta_t) + \varepsilon / 4  \\
    & \leq \min_{\vtheta \in \Theta} \max_{i \in [k]} L_{i}(\vtheta) + \varepsilon
\end{align*}
where the first inequality follows from \eqref{eq:concentration_probability}, the second inequality follows from \eqref{eq:alpha_player} with $u=e_i$, and the last inequality follows from \eqref{eq:w_player}. 

\end{proof}

\subsection{Additional Results}

\begin{proof}[Proof of Example \ref{ex:ogd}]
    \citep{shalev2012online} implies that online gradient descent with learning rate $\eta = \frac{1}{L \sqrt{2T}}$ is $BL \sqrt{2}$-low-regret where $L$ is the Lipschitz constant of $l(f_w(x), y)$ and $\norm{w}_2 \leq B$. By assumption, $B=1$. To bound on the Lipschitz constant of $l(f_w(x), y)$ wrt $w$ over $\cw$, it suffices to bound the euclidean norm of the gradient wrt $w$. Therefore,
    \begin{align*}
        \norm{\nabla_w l(f_w(x), y)}_2 & =  \norm{2|w \cdot x - y| x}_2 \\
        & \leq 2 |w \cdot x - y| \norm{ x}_2 \\
        & \leq  2 |w \cdot x - y| \\
        & \leq 2 ( \norm{w}_2 \norm{x}_2 + |y|) \\
        & \leq 4,
    \end{align*}
    which yields a bound of $4$, completing the proof.
\end{proof}

\section{Algorithms}
\label{app:algorithms}

\begin{algorithm}[t]
\caption{AutoMixAlign Reweighting (AMA-R)}
\label{alg:ama_r}
\begin{algorithmic}[1]
\STATE \textbf{Inputs:} Task datasets $\sD_1, \dots, \sD_k$, specialist parameters $\vtheta_1, \dots, \vtheta_k$, batch size $b$, smoothing parameter $c\in(0,1)$, training steps $T$
\STATE \textbf{Outputs:} Trained model parameters $\vtheta$
\STATE For all tasks, compute $\sL(\vtheta_i, z)$ for each $z \in \sD_i$, and add them to a new column in $\sD_i$ \label{line:precompute}
\STATE Initialize model parameters $\vtheta$
\STATE $\bm{q}^{(1)} \gets (\frac{1}{k},\dots,\frac{1}{k})$
\FOR{$t = 1, \dots, T$}
    \STATE $\alpha^{(t)}_i \gets (1-c)q^{(t)}_i + c\frac{1}{k}, \forall i$
    \STATE $\sT \sim \text{Uniform}(1, \dots, 
    k, b)$
    \STATE $\sB \gets \{z \sim \sD_j : \forall j \in \sT\}$
    \STATE $\sB_i \gets \sB \cap \sD_i$ for each task $i$ 
    \STATE Update $\bm{q}^{(t)}$ using exponentiated gradient ascent:
    $$
    {q}^{(t+1)}_i \gets {q}^{(t)}_i \cdot \exp\left\{ \frac{1}{|\sB_i|} \sum_{z\in\sB_i}\max\{ \mathcal L(\vtheta, z)-\mathcal L(\vtheta_i, z), 0\}\right\}
    $$
    $$
    {q}^{(t+1)}_i \gets \frac{{q}^{(t+1)}_i}{\sum_{j=1}^k {q}^{(t+1)}_j}
    $$
    \STATE Perform one step of gradient descent on the loss 
    $$
    \sum_{i=1}^k \alpha_i\frac{1}{|\sB_i|}\sum_{z\in\sB_i} \mathcal \max\{ \sL(\vtheta, z)-\mathcal \sL(\vtheta_i, z), 0\}
    $$

\ENDFOR
\end{algorithmic}
\end{algorithm}

\begin{algorithm}[t]
\caption{Standard Optimization (Standard)}
\label{alg:standard}
\begin{algorithmic}[1]
\STATE \textbf{Inputs:} Task datasets $\sD_1, \dots, \sD_k$, batch size $b$, training steps $T$
\STATE \textbf{Outputs:} Trained model parameters $\vtheta$
\STATE Initialize model parameters $\vtheta$
\FOR{$t = 1, \dots, T$}
    \STATE Sample minibatch $\sB$ containing $b$ samples drawn uniformly at random from $\bigcup_{i=1}^k \sD_i$
    \STATE Perform one step of gradient descent on the loss 
    $
    \frac{1}{|\sB|} \sum_{z\in\sB} \sL(\vtheta, z)
    $
\ENDFOR
\end{algorithmic}
\end{algorithm}

\begin{algorithm}[t]
\caption{Standard Optimization with Uniform Task Sampling (Standard Uniform)}
\label{alg:standard_uniform}
\begin{algorithmic}[1]
\STATE \textbf{Inputs:} Task datasets $\sD_1, \dots, \sD_k$, batch size $b$, training steps $T$
\STATE \textbf{Outputs:} Trained model parameters $\vtheta$
\STATE Initialize model parameters $\vtheta$
\FOR{$t = 1, \dots, T$}
    \STATE $\sT \sim \text{Uniform}(1,\dots,k, b)$
    \STATE $\sB \gets \{z \sim \sD_j : \forall j \in \sT\}$
    \STATE Perform one step of gradient descent on the loss 
    $
    \frac{1}{|\sB|} \sum_{z\in\sB} \sL(\vtheta, z)
    $
\ENDFOR
\end{algorithmic}
\end{algorithm}



\clearpage



\clearpage

\begin{table*}[h]
    \centering
    \resizebox{\textwidth}{!}{
    \begin{tabular}{l|c|c|c|c|c|c}
    \hline
    \rowcolor[HTML]{EFEFEF} 
    \textbf{Model} & \textbf{IFEval (\%)}   $\uparrow$ & \textbf{Alpaca Eval (\%)}   $\uparrow$ & \textbf{MBPP (\%)}   $\uparrow$ & \textbf{HumanEval (\%)}   $\uparrow$ & \textbf{Toxigen (\%)}   $\uparrow$ & \textbf{Average (\%)}   $\uparrow$\\ \hline
    Zephyr-7b-sft-full & 35.30 & 8.64  & 49.90 & 50.46 & 51.60  & 39.18 \\ \hline\hline
    Standard              & 38.63 & 11.32 & 50.00 & \textbf{54.88} & 99.91  & 50.95 \\ 
    AMA-S         & 43.62 & 13.04 & \textbf{55.76} & 53.66 & 99.80  & \textbf{53.18} \\ \hline
    AMA-S Uniform                           & 42.51 &  \textbf{13.90} & 52.16 & 51.83 & 99.80  & 52.05 \\ 
    AMA-S DoReMi     & \textbf{44.81} &  10.97 & 45.32 & 53.05 & \textbf{99.99}  & 50.70 \\ 
    AMA-S Unclipped                    & 42.51 &  12.48 & 48.56 & 53.66 & 99.84  & 51.42 \\ 
    \hline
    \end{tabular}
    }
    \caption{AMA-S ablations for Setup 3, Helpfulness (Chatbot Arena 2024) + Coding (CodeUltraFeedback) + Harmlessness (SafeRLHF).}
    \label{tab:ama_s_ablation}
\end{table*}

\begin{table*}[h]
    \centering
    \resizebox{\textwidth}{!}{
    \begin{tabular}{l|c|c|c|c|c|c}
    \hline
    \rowcolor[HTML]{EFEFEF} 
    \textbf{Model} & \textbf{IFEval (\%)}   $\uparrow$ & \textbf{Alpaca Eval (\%)}   $\uparrow$ & \textbf{MBPP (\%)}   $\uparrow$ & \textbf{HumanEval (\%)}   $\uparrow$ & \textbf{Toxigen (\%)}   $\uparrow$ & \textbf{Average (\%)}   $\uparrow$\\ \hline
    Zephyr-7b-sft-full & 35.30 & 8.64  & 49.90 & 50.46 & 51.60  & 39.18 \\ \hline\hline
    Standard        & 38.63 & 11.32 & 50.00 & 54.88 & \textbf{99.91}  & 50.95 \\ 
    AMA-R           & 48.43 & 17.77 & 54.32 & \textbf{55.49} & 95.91  & \textbf{54.38} \\\hline
    AMA-R Uniform   & 45.47 & 9.69 & \textbf{54.68} & 53.66 & 99.54  & 52.61 \\ 
    AMA-R DoReMi    & \textbf{49.35} & 14.02 & 53.95 & 53.66 & 99.47 & 54.09\\
    AMA-R Unclipped & 47.50 & \textbf{18.56} & 48.56 & 51.83 & 96.59 & 52.61 \\ \hline
    \end{tabular}
    }
    \caption{AMA-R ablations for Setup 3, Helpfulness (Chatbot Arena 2024) + Coding (CodeUltraFeedback) + Harmlessness (SafeRLHF).}
    \label{tab:ama_r_ablation}
\end{table*}

\section{Additional Experiments}
\label{app:ablations}

In this appendix, we first present AMA ablations to understand how adaptive task weighting and excess loss clipping affect performance. 
Next, since our model merging baseline in our main experiments computes a uniform average of specialist parameters, we present ablations on model merging with non-uniform parameter averaging.
We focus on Setup 3 (Helpfulness + Coding + Harmlessness).
\subsection{AMA Ablations}

In this appendix, we run additional ablation experiments to better understand how adaptive task weighting and excess loss clipping affect AMA-R and AMA-S performance. 
We consider the following ablations
\begin{enumerate}
    \item \textbf{AMA-R/S Uniform} (Algorithm~\ref{alg:ama_s_uniform}). Fix task weights at $1/k$ throughout training.
    \item \textbf{AMA-R/S DoReMi} (Algorithm~\ref{alg:ama_s_doremi}). Fix task weights at the average weights obtained throughout AMA-S training, as is done in DoReMi.~\citep{xie2024doremi}.\footnote{We use this baselines to determine if a fixed, non-uniform task weighting used in DoReMi is superior to a fixed, uniform weighting. We note that DoReMi determines task weights by optimizing the \textit{unclipped} excess loss, whereas we are optimizing the clipped excess loss. Despite this difference, this baseline captures the spirit of DoReMi.} 
    \item \textbf{AMA-R/S  Unclipped} (Algorithm~\ref{alg:ama_s_unclipped}). Remove the inner $\max$ in Eq.~\ref{eq:ama} so that the excess loss is no longer clipped at 0 for model updates.
\end{enumerate}
Note that Algorithms~\ref{alg:ama_s_uniform},~\ref{alg:ama_s_doremi}, and~\ref{alg:ama_s_unclipped} follow AMA-S. We omit pseudocode for AMA-R ablations because they are analogous to the pseudocode for AMA-S ablations.

We present AMA-S ablation results in Table~\ref{tab:ama_s_ablation} and AMA-R ablation results in Table~\ref{tab:ama_r_ablation}
All AMA-S ablations yield lower average performance than AMA-S, and all AMA-R ablations yield lower average performance than AMA-R, indicating that fixed weights lead to suboptimal task balancing and that excess loss clipping is critical for good performance.

AMA-S Uniform scores higher on average than Standard whereas AMA-S DoReMi and Standard have similar average scores. 
This result emphasizes that while there exist fixed task weights that can improve over Standard (uniform weights in this setup), other fixed task weights may not yield improvements.
Since it is not obvious how to determine the ideal fixed weighting, adaptive weighting is critical.

AMA-R Unclipped and AMA-S Unclipped degrade performance on MBPP w.r.t.\@ the base model.
Without clipping, the generalist model optimizes beyond the losses achieved by specialist models.
Since specialist model losses correspond to strong performance, further optimization causes our generalist model to overfit to CodeUltraFeedback.

We now show this phenomenon empirically.
In Figure~\ref{fig:codeuf_excess_loss}, we show the (unclipped) excess loss on CodeUltraFeedback during AMA-S Unclipped training. 
After 0.5 epochs of training, AMA-S Unclipped achieves negative excess loss, indicating that it has optimized beyond the specialist model losses.
Since we save model checkpoints every 0.5 epochs (as described in Appendix~\ref{app:eval}), \textit{all} model checkpoints were optimized beyond the specialist model.

\begin{figure}
    \centering
    \includegraphics[width=0.5\linewidth]{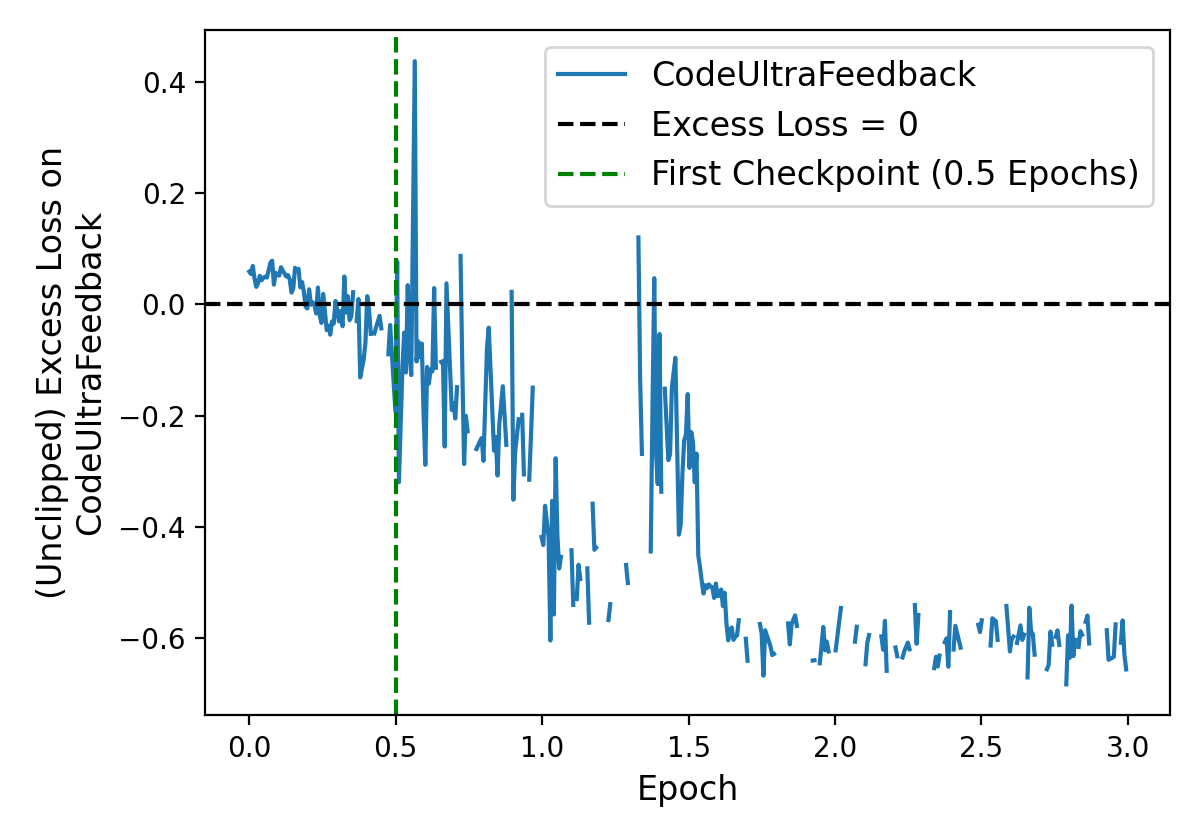}
    \caption{Excess loss on CodeUltraFeedback during AMA-S Unclipped training in Setup 3. We save model checkpoints every 0.5 epochs, and the green dashed line indicates where we save our first model checkpoint. The horizontal dashed line indicates zero excess loss. After 0.5 epochs of training, AMA-S Unclipped achieves negative excess loss, indicating that it has optimized beyond the specialist model losses. The gaps in the excess loss curve arise because in some updates, no CodeUltraFeedback samples were sampled. }
    \label{fig:codeuf_excess_loss}
\end{figure}

\begin{algorithm}[h]
\caption{AutoMixAlign Resampling Uniform (AMA-S Uniform)}
\label{alg:ama_s_uniform}
\begin{algorithmic}[1]
\STATE \textbf{Inputs:} Task datasets $\sD_1, \dots, \sD_k$, specialist parameters $\vtheta_1, \dots, \vtheta_k$, batch size $b$, training steps $T$
\STATE \textbf{Outputs:} Trained model parameters $\vtheta$
\STATE For all tasks, compute $\sL(\vtheta_i, z)$ for each $z \in \sD_i$, and add them to a new column in $\sD_i$ \label{line:precompute}
\STATE Initialize model parameters $\vtheta$
\FOR{$t = 1, \dots, T$}
    \STATE $\sT \sim \text{Uniform}(1,\dots,k, b)$
    \STATE $\sB \gets \{z \sim \sD_j : \forall j \in \sT\}$
    \STATE Update $\vtheta$ using one step of gradient descent on the loss 
    $$
    \frac{1}{|\sB|} \sum_{z\in\sB} \mathcal \max\{\sL(\vtheta, z)-\mathcal \sL(\vtheta_i, z), 0\}
    $$
    
\ENDFOR
\end{algorithmic}
\end{algorithm}

\begin{algorithm}[h]
\caption{DoReMi}
\label{alg:ama_s_doremi}
\begin{algorithmic}[1]
\STATE \textbf{Inputs:} Task datasets $\sD_1, \dots, \sD_k$, specialist parameters $\vtheta_1, \dots, \vtheta_k$, batch size $b$, training steps $T$
\STATE \textbf{Outputs:} Trained model parameters $\vtheta$
\STATE For all tasks, compute $\sL(\vtheta_i, z)$ for each $z \in \sD_i$, and add them to a new column in $\sD_i$ \label{line:precompute}
\STATE Initialize model parameters $\vtheta$
\STATE Initialize $\valpha$ to be the average task weighting achieved during a training run of AMA (Algorithm~\ref{alg:ama}).
\FOR{$t = 1, \dots, T$}
    \STATE $\sT \sim \text{Multinomial}(\alpha_1,\dots,\alpha_k, b)$
    \STATE $\sB \gets \{z \sim \sD_j : \forall j \in \sT\}$
    \STATE Update $\vtheta$ using one step of gradient descent on the loss 
    $$
    \frac{1}{|\sB|} \sum_{z\in\sB} \mathcal \max\{\sL(\vtheta, z)-\mathcal \sL(\vtheta_i, z), 0\}
    $$
    
\ENDFOR
\end{algorithmic}
\end{algorithm}

\begin{algorithm}[h]
\caption{AMA-S Unclipped}
\label{alg:ama_s_unclipped}
\begin{algorithmic}[1]
\STATE \textbf{Inputs:} Task datasets $\sD_1, \dots, \sD_k$, specialist parameters $\vtheta_1, \dots, \vtheta_k$, batch size $b$, smoothing parameter $c\in (0,1)$, training steps $T$
\STATE \textbf{Outputs:} Trained model parameters $\vtheta$
\STATE For all tasks, compute $\sL(\vtheta_i, z)$ for each $z \in \sD_i$, and add them to a new column in $\sD_i$ \label{line:precompute}
\STATE Initialize model parameters $\vtheta$
\STATE $\bm{q}^{(1)} \gets (\frac{1}{k},\dots,\frac{1}{k})$
\FOR{$t = 1, \dots, T$}
    \STATE $\alpha^{(t)}_i \gets (1-c)q^{(t)}_i + c\frac{1}{k}, \forall i$
    \STATE $\sT \sim \text{Multinomial}(\alpha^{(t)}_1, \dots, 
    \alpha^{(t)}_k, b)$
    \STATE $\sB \gets \{z \sim \sD_j : \forall j \in \sT\}$
    \STATE $\sB_i \gets \sB \cap \sD_i$ for each task $i$ 
    \STATE Update $\valpha$ using one step of EXP3 ~\cite{auer2002finite}:
    $$
    q^{(t+1)}_i \gets q^{(t)}_i \cdot \exp\left\{ \frac{1}{q^{(t)}_i}\frac{1}{|\sB_i|} \sum_{z\in\sB_i} \max\{\mathcal L(\vtheta, z)-\mathcal L(\vtheta_i, z), 0\}\right\}
    $$
    $$
    q^{(t+1)}_i \gets \frac{q^{(t+1)}_i}{\sum_{j=1}^k q^{(t+1)}_j}
    $$
    \STATE Update $\vtheta$ using one step of gradient descent on the loss 
    $$
    \frac{1}{|\sB|} \sum_{z\in\sB} \sL(\vtheta, z)-\mathcal \sL(\vtheta_i, z)
    $$
    
\ENDFOR
\end{algorithmic}
\end{algorithm}

\clearpage
\begin{table}[h]
\centering
\resizebox{\textwidth}{!}{
\begin{tabular}{l|c|c|c|c|c}
\hline
\rowcolor[HTML]{EFEFEF}
\textbf{Model} & \textbf{IFEval (\%)}   $\uparrow$ & \textbf{Alpaca Eval (\%)}   $\uparrow$ & \textbf{MBPP (\%)}   $\uparrow$ & \textbf{HumanEval (\%)} $\uparrow$ & \textbf{Average (\%)}   $\uparrow$\\ \hline
    Zephyr-7b-sft-full & 35.30 &  8.64 & 49.90 & 50.46 & 36.08 \\ \hline
    Chatbot Arena 2024 Specialist & 43.44 & 16.05 & 44.60 & 50.61 & 38.68 \\ 
    CodeUltraFeedback  Specialist & 35.86 &  9.93 & 52.16 & 57.32 & 38.82 \\ \hline\hline
 $0.25\vtheta_\text{Chatbot Arena 2024} + 0.75\vtheta_\text{CodeUltraFeedback}$   & 37.15 & 8.48 & 50.36 & 54.88 & 37.72  \\
 $0.50\vtheta_\text{Chatbot Arena 2024} + 0.50\vtheta_\text{CodeUltraFeedback}$   & 42.14 & 11.51 & 48.56 & \textbf{55.49} & 39.43 \\ 
 $0.75\vtheta_\text{Chatbot Arena 2024} + 0.25\vtheta_\text{CodeUltraFeedback}$    & 44.73 & 17.56 & 48.20 & 51.22  &  40.43 \\
\hline
    AMA-R        & 42.51 & \textbf{18.15} & \textbf{51.44} & 54.88 & \textbf{41.75} \\
    AMA-S        & \textbf{43.44} & 15.61 & 51.08 & 50.61 & 40.19 \\ \hline
\end{tabular}}
\caption{Weighted Model Averaging results for Setup 2, Chatbot Arena 2024 (Helpfulness) + CodeUltraFeedback (Coding)}
\label{tab:model_merging_results_1}
\end{table}

\begin{table}[h]
\centering
\resizebox{\textwidth}{!}{
\begin{tabular}{l|c|c|c|c}
\hline
\rowcolor[HTML]{EFEFEF}
\textbf{Model } & \textbf{IFEval (\%)}   $\uparrow$ & \textbf{Alpaca Eval (\%)}   $\uparrow$ & \textbf{Toxigen (\%)}   $\uparrow$ & \textbf{Average (\%)}   $\uparrow$\\
\hline
    Zephyr-7b-sft-full     & 35.30 & 8.64  & 51.60  & 31.85  \\ \hline
    UltraFeedback Specialist          & 42.51 & 17.59 & 36.10  & 32.07  \\ 
    SafeRLHF Specialist        & 29.02 & 6.66  & 100.00 & 45.23  \\ \hline\hline
 $0.25\vtheta_\text{UltraFeedback} + 0.75\vtheta_\text{SafeRLHF}$   & 36.78 & 8.48 & \textbf{99.99} & 48.42 \\
 $0.50\vtheta_\text{UltraFeedback} + 0.50\vtheta_\text{SafeRLHF}$   & 39.37 & 12.21 & 99.96  & 50.51  \\ 
 $0.75\vtheta_\text{UltraFeedback} + 0.25\vtheta_\text{SafeRLHF}$    & 43.99 & 15.69 & 78.17 & 45.95     \\
\hline
    AMA-R         & \textbf{44.55} &  16.21 & 99.74  & 53.50  \\
    AMA-S              & \textbf{44.55} & \textbf{16.88} & \textbf{99.99}  & \textbf{53.81}  \\ \hline
\end{tabular}}
\caption{Weighted Model Averaging results for Setup 2, UltraFeedback (Helpfulness) + SafeRLHF (Harmlessness)}
\label{tab:model_merging_results_2}
\end{table}

\begin{table}[h!]
\centering
\resizebox{\textwidth}{!}{
\begin{tabular}{l|c|c|c|c|c|c}
\hline
\rowcolor[HTML]{EFEFEF} 
\textbf{Model} & \textbf{IFEval (\%)}   $\uparrow$ & \textbf{Alpaca Eval (\%)}   $\uparrow$ & \textbf{MBPP (\%)}   $\uparrow$ & \textbf{HumanEval (\%)}   $\uparrow$ & \textbf{Toxigen (\%)}   $\uparrow$ & \textbf{Average (\%)}   $\uparrow$\\ \hline
Zephyr-7b-sft-full & 35.30 & 8.64  & 49.90 & 50.46 & 51.60  & 39.18 \\ \hline
Chatbot Arena 2024 Specialist & 43.44 & 16.05 & 44.60 & 50.61 & 69.77  & 44.89 \\ 
CodeUltraFeedback Specialist & 35.86 & 9.93  & 52.16 & 57.32 & 76.31  & 46.31 \\ 
SafeRLHF Specialist & 29.02 & 6.66  & 48.56 & 50.61 & 100.00 & 46.97 \\ \hline\hline
 $\frac{1}{6}\vtheta_\text{Chatbot Arena 2024} + \frac{5}{12}(\vtheta_\text{CodeUltraFeedback} + \vtheta_\text{SafeRLHF})$   & 35.86 &  7.58 & 50.00 & 54.27 & 96.50 & 48.84  \\
 $\frac{1}{3}\vtheta_\text{Chatbot Arena 2024} + \frac{1}{3}(\vtheta_\text{CodeUltraFeedback} + \vtheta_\text{SafeRLHF})$    & 42.70 & 12.28 & 52.16 & 53.66 & 96.70  & 51.50 \\ 
$\frac{5}{12}\vtheta_\text{Chatbot Arena 2024} + \frac{1}{6}(\vtheta_\text{CodeUltraFeedback} + \vtheta_\text{SafeRLHF})$  & 43.81 & 17.56 & 48.20 & 51.22 & 70.51 & 46.26 \\
\hline
AMA-R         & \textbf{48.43} & \textbf{17.77} & 54.32 & \textbf{55.49} & 95.91  & \textbf{54.38} \\

AMA-S        & 43.62 & 13.04 & \textbf{55.76} & 53.66 & \textbf{99.80}  & 53.18 \\ \hline
\end{tabular}}
\caption{Weighted Model Averaging results for Setup 3, Chatbot Arena 2024 (Helpfulness) + CodeUltraFeedback (Coding) + SafeRLHF (Harmlessness)}
\label{tab:model_merging_results_3}
\end{table}

\subsection{Model Merging Ablations}

The model merging baseline included in our main experiments uniformly averages the parameters of specialist models. To ablate the effects of parameter weighting with model merging, we investigate non-uniform merging methods.
For the setups 1 and 2, we merge specialist model parameters as
$$\vtheta_{\text{merged}} = \lambda \vtheta_{\text{specialist 1}} + (1-\lambda) \vtheta_{\text{specialist 2}}$$
for $\lambda \in \{0.25, 0.5, 0.75\}$. 
For setup 3, we use 
$$\vtheta_\text{merged} = \lambda\vtheta_{\text{base}} + \frac{(1-\lambda)}{2}(\vtheta_{\text{specialist 1}} + \vtheta_{\text{specialist 2}})$$ 
for $\lambda \in \{1/6, 1/3, 5/12\}$ with Chatbot Arena specialist model as the base model $\vtheta_\text{base}$.

We present results in Tables~\ref{tab:model_merging_results_1},~\ref{tab:model_merging_results_2}, and~\ref{tab:model_merging_results_3}.
In setups 2 and 3, both AMA algorithms score higher on average than all model merging methods. In setup 1, AMA-R has the highest average score, though $0.75\vtheta_\text{Chatbot Arena 2024} + 0.25\vtheta_\text{CodeUltraFeedback}$ has a higher average score than AMA-S. However, AMA-S improves performance on MBPP whereas this merging method degrades performance on MBPP.

\section{Comparing AMA with Existing DRO Algorithms}
\label{app:ama_vs_doremi}

In this section, we discuss differences between AMA and several closely-related DRO-style algorithms: GroupDRO~\cite{sagawa2019distributionally},  ODM~\cite{albalak2023efficient}, Sheared LLama~\cite{xia2023sheared}, ExcessMTL~\cite{he2024robust}, and DoReMi~\cite{xie2024doremi}.
GroupDRO minimizes the worst-case loss across multiple tasks (\textit{i.e.}, groups) by reweighting the loss function to put more weight on tasks with larger loss.
DoReMi and ExcessMTL are DRO-based pretraining algorithms that reweight tasks by prioritizing those with the largest clipped excess loss (while optimizing the usual loss).
ODM and Sheared Llama are online data mixing algorithms that increases the probability of sampling data from tasks with large loss or excess loss.
AMA is similar to DoReMi and GroupDRO because all three introduce a minimax optimization problem used to optimize the worst-case loss or excess loss. Sheared LLama and ODM are similar to AMA in that they are adapt the probability of sampling data from each task during training.
However, AMA differs in several respects:
\begin{enumerate}
    \item AMA focuses on preference alignment, whereas all of these prior works focus on pretraining. GroupDRO additionally considers NLP and vision tasks.
    \item AMA optimizes the clipped excess loss for both task weight and model updates. DoReMi and Sheared LLama optimize the usual loss for model updates. GroupDRO and ODM optimize the usual loss for both task weight and model updates.
    \item AMA \textit{adaptively} updates task weights \textit{during} model optimization, whereas DoReMi learns a \textit{fixed} task weighting prior to model optimization. 
    \item AMA uses single-task specialist models to compute excess losses. DoReMi uses a single multitask model (called the \textit{proxy model}) to compute excess losses. ExcessMTL uses a diagonal approximation of the generalist's Hessian. Sheared LLaMA approximates these losses with single aggregate reference loss computed as the average loss computed over a validation dataset: $\sL(\sD_i^\text{val}, \vtheta) = \frac{1}{|\sD_i^\text{val}|}\sum_{x\in\sD_i^\text{val}}\sL(x,\vtheta)$. Thus, Sheared LLaMA will underestimate the excess loss of easy samples and overestimate the excess loss of hard samples. AMA computes references loss for each sample individually and thus more accurately estimates excess losses.
    \item We introduce two AMA algorithms---one that uses reweighting (AMA-R) and another that uses resampling (AMA-S)---whereas GroupDRO and DoReMi only introduce a reweighting algorithm. We stress that while \citet{sagawa2019distributionally} provide a convergence result for GroupDRO, their proof only applies to the reweighting-based AMA-R algorithm and does not extend to our resampling-based AMA-S algorithm. A primary contribution of this work is that we provide a new theoretical argument to prove convergence for AMA-S.
\end{enumerate}

\section{Excess Loss}
\label{app:excess_loss}
In this section, we provide background motivation for the using the excess loss in AMA. The following discussion is largely summarized from \citet{he2024robust}.

Consider a supervised learning setting where we want to train a model $\vtheta \in \Theta$ to predict the label $y \in \sY$ given input data $x\in\sX$ where samples $(x,y)$ are drawn from some distribution $\sP$. 
%
%
Given loss function $\sL$, the risk of $\vtheta$ is
\begin{equation}
        \sL(\vtheta) = \mathbb{E}_{(x,y)\sim \sD}[\sL(\vtheta, x, y)]. \\
\end{equation}
The risk can be decomposed as
\begin{equation}
\sL(\vtheta) =  \underbrace{\underbrace{\sL(\vtheta) - \sL(\vtheta^*_\Theta)}_{\text{Estimation error}} + 
\underbrace{\sL(\vtheta^*_\Theta) - \sL(\vtheta^*)}_{\text{Approximation error}}}_{\text{Excess risk }} + 
\underbrace{\sL(\vtheta^*)}_{\text{Bayes error}},
\end{equation}
where $\vtheta^*_\Theta = \argmin_{\vtheta\in\Theta} \mathbb{E}_{(x,y)\sim \sP}[\sL(\vtheta, x, y)]$ is the optimal model in model family $\Theta$ and $\vtheta^* = \argmin_{\vtheta} \mathbb{E}_{(x,y)\sim \sP}[\sL(\vtheta, x, y)]$ is the optimal model in any model family. 
Approximation error arises from our choice of model family, but when working with expressive function approximators like LLMs, the approximation error is generally very small. 
The Bayes error arises from label noise and is irreducible. 
The \textit{excess risk} $\sL(\vtheta) - \sL(\vtheta^*)$ denotes the gap between our model's loss and the loss of the Bayes optimal model (\textit{i.e.} the minimum loss achievable). 
Since excess loss removes the loss contribution of label noise and approximation error is effectively 0, it can be interpreted as our model's ``room for improvement.''

The \textit{excess loss} $\sL(\vtheta,x,y) - \sL(\vtheta^*, x,y)$ quantifies the room for improvement on a particular sample.
Samples with high excess loss are learnable ($\sL(\vtheta^*, x,y)$ is small) but our model has not yet learned them. 
Samples with low excess loss are high entropy samples that are less learnable ($\sL(\vtheta^*,x,y)$ is large) or samples that our model has already learned ($\sL(\vtheta, x,y) \approx \sL(\vtheta^*, x,y)$). 

The excess loss is particularly useful in minimax optimization problems where we want to minimize the worst-case loss over different tasks, as is the case with AMA. 
Let $\sL_i(\vtheta) = \frac{1}{|\sD_i|}\sum_{(x,y)\in\sD_i}\sL(\vtheta,x,y)$ denote the loss for task $i$, and consider the following optimization problem:
\begin{equation}
    \label{eq:minimax_loss}
    \min_\vtheta \max_{i\in[k]}\sL_i(\vtheta)
\end{equation}
Suppose the task $i$ has zero excess loss but has a larger loss than task $j$. Such a scenario can arise if task $i$ is difficult to learn.
In this case, optimization problem~\ref{eq:minimax_loss} will focus on optimizing task $i$ even though it cannot minimize its loss any further and make no progress on task $j$.
Minimizing worst-case excess loss prevents optimization from focusing on difficult tasks at the expense of easier tasks.

\section{Future Directions}
\label{app:future_directions}

\noindent \textbf{Partitioning data into tasks.} We relied on intuitive notions of what a ``task'' represents, \textit{e.g.} datasets designed to improve helpfulness fall under the helpfulness task.
However, this intuitive approach is not necessarily optimal, and future work should study how to partition data into tasks to maximize performance in DRO-based algorithms.
For instance, one could automatically partition data into tasks based on classification from another LLM.  \newline

\noindent \textbf{Applying the methodology to other layers of the post-training stack.} We mainly focused on the DPO layer of the post-training learning layer. In principle, our AMA methodology can be applied to others, including reward modeling, supervised-fine-tuning and on-policy learning algorithms such as PPO. We are excited to see what benefit this methodology can bring to these other layers.

\begin{table}[]
    \centering
    \begin{tabular}{l|l}
    \hline     \rowcolor[HTML]{EFEFEF} 
        Hyperparameter & Value \\
        \hline
        Learning rate & $5\cdot 10^{-7}$ \\
        Batch size & 256 \\
        Per-device batch size & 8 \\
        Gradient accumulation steps & 4 \\
        Optimizer & Adam with $\beta_1=0.9,\beta_2 = 0.999, \varepsilon = 10^{-8}$\\
        Learning rate scheduler & cosine \\
        Learning rate scheduler warmup ratio & 0.1 \\
        Number of training epochs & 3 \\
        \hline
        AMA task weight learning rate $\eta$ & 
        1 \\
        AMA smoothing parameter $c$ & 0.1 \\
        \hline
    \end{tabular}
    \caption{Hyperparameters fixed across our main experiments.}
    \label{tab:hyperparameters}
\end{table}

\begin{table}[]
    \centering
    \resizebox{\textwidth}{!}{
    \begin{tabular}{l|l}
    \hline     \rowcolor[HTML]{EFEFEF} 
      Dataset Name & Link \\\hline
      UltraFeedback   & \url{https://huggingface.co/datasets/HuggingFaceH4/ultrafeedback_binarized} \\
      Chatbot Arena 2024   & \url{https://huggingface.co/datasets/lmsys/chatbot\_arena\_conversations}\\
      CodeUltraFeedback & \url{https://huggingface.co/datasets/coseal/CodeUltraFeedback}\\
      PKU-SafeRLHF & \url{https://huggingface.co/datasets/PKU-Alignment/PKU-SafeRLHF}\\\hline
    \end{tabular}
    }
    \caption{Datasets used in our experiments.}
    \label{tab:datasets}
\end{table}

\begin{figure}[h]
    \centering
    \includegraphics[width=0.7\linewidth]{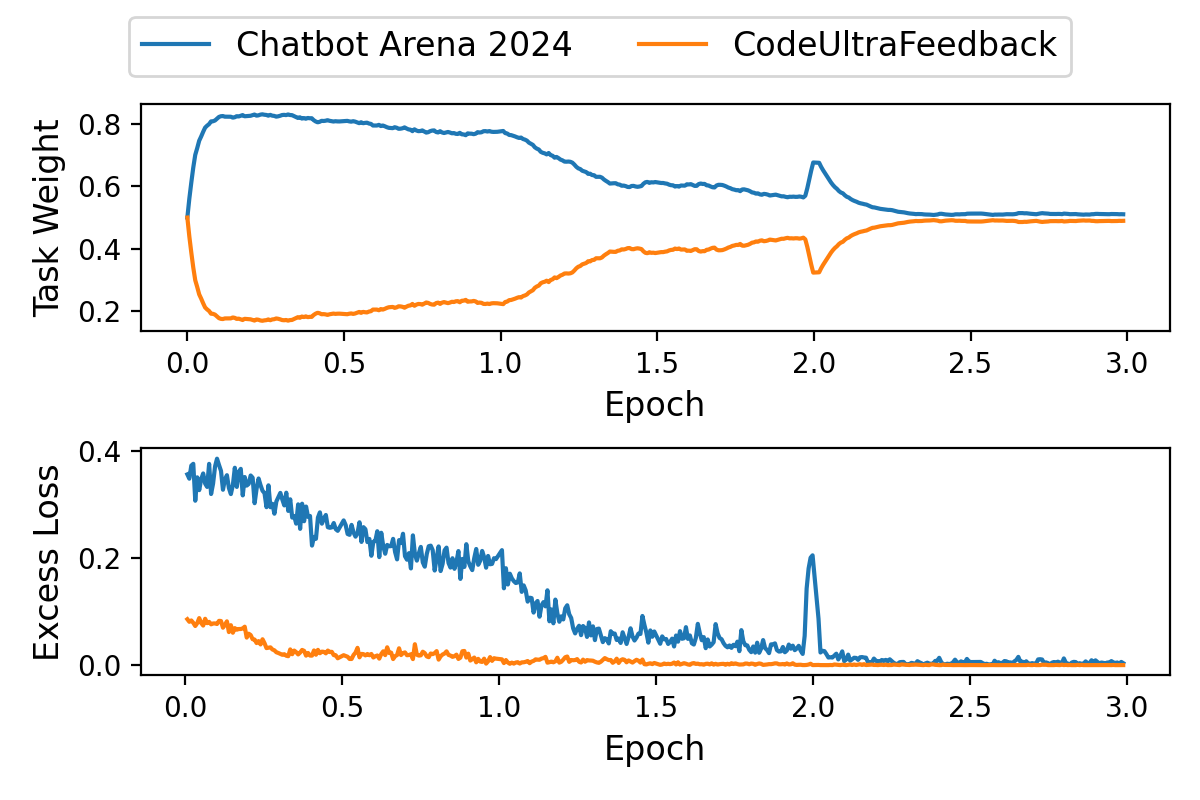}
    \caption{Excess losses and task weights for AMA-R in Setup 1, Helpfulness (Chatbot Arena 2024) + Coding (CodeUltraFeedback).}
    \label{fig:chatbot_codeuf_metrics_ama_r}
\end{figure}
\section{Training Details}

\label{app:training}
We provide DPO and AMA hyperparameters for our main experiments in Table~\ref{tab:hyperparameters} and links to datasets in Table~\ref{tab:datasets}. We built our codebase on top of the alignment-handbook repository\footnote{\url{https://github.com/huggingface/alignment-handbook}} and run all training jobs on 8 NVIDIA A100 80GB GPUs.
AMA trains $k$ specialist models and one generalist model and thus requires $2\sum_{i=1}^k|D_i|$ model updates in total, twice as many required to train a multitask model using standard optimization approaches.






\section{Evaluation}
\label{app:eval}
In this appendix, we describe our model checkpoint selection method and how we evaluate the selected checkpoint.

\subsection{Model Selection}

We save model checkpoints every 0.5 epochs. 
For each checkpoint, compute a confidence interval for the evaluation accuracy across all tasks (+/- 1 standard error):
\begin{equation}
    CI = [\hat p - \sigma, \hat p + \sigma], \quad \sigma = \sqrt{\sum_{i=1}^k\sigma_i^2}, \qquad \sigma_i = \sqrt{\frac{\hat p_i (1-\hat p_i)}{n_i}}
\end{equation}
Here, $\hat p$ is the evaluation accuracy, $\sigma_i$ is the standard error for task $i$, $n_i$ is the number of samples in the eval dataset for task $i$, and $\sigma$ is the standard error for all tasks combined.
We then select the earliest checkpoint whose CI overlaps with the CI of the checkpoint with the maximum evaluation accuracy.

Researchers often look at the evaluation loss for model selection (\textit{e.g,}, selecting the checkpoint saved just before the model begins overfitting), but this metric can be misleading in multi-task settings. 
The losses for different tasks are often on different scales (\textit{i.e.}, the excess loss of task 1 may always be larger than that of task 2), and it's unclear how we should adjust for these differences in loss scale. 
Since task accuracy, the fraction of samples in the evaluation dataset for which $\pi(y_w|x) > \pi(y_l|x)$, is always on a scale from 0 to 1, we instead look at task accuracies to select model checkpoints. 

\subsection{Benchmarks}

We evaluate a single model for each algorithm and experimental setup and base our evaluation on a subset of the benchmarks used by \cite{ivison2024UnpackingDA} described below.
To run AlpacaEval, we use the original \texttt{alpaca\_eval} codebase~\cite{alpaca_eval, dubois2024length}.\footnote{\url{https://github.com/tatsu-lab/alpaca_eval}}
For the remaining benchmarks, we use the \texttt{open-instruct} codebase.\footnote{\url{https://github.com/allenai/open-instruct}}. 
Since we use the Zephyr 7b SFT Full model~\citep{tunstall2023zephyr} as the base policy in all experiments, we use the Zephyr chat template in all evaluations.
%

\begin{itemize}
    \item \textbf{IFEval (Helpfulness)}~\citep{zhou2023instruction}: We use the default setup and report the “Strict Accuracy” metric at the prompt level, \textit{i.e.} the percentage of prompts for which a model's output satisfies all verifiable instructions.
    
    \item \textbf{Alpaca Eval 2.0 (Helpfulness)}~\citep{dubois2024length}: We use the AlpacaEval 2.0 configuration provided for zephyr-7b-alpha-ExPO, reporting the length-controlled winrate.
    We use a temperature of $0.7$, nucleus sampling with $p=0.9$ (top-p), and a top-k value of 50. To control repetition, we applied a presence penalty of $0.1$ and a frequency penalty of $0.1$. report the length-controlled win rate.
    We allow the evaluated model to generate up to 2048 tokens.
    %
    

    \item \textbf{MBPP (Coding)}~\citep{austin2021program}: We report pass@10 accuracy and use a sampling temperature of 0.8. 
    \item \textbf{HumanEval (Coding)}~\citep{chen2021evaluating}): In addition to the original HumanEval prompts, we use the prompts provided by HumanEvalPack~\cite{muennighoff2023octopack}. Just as with MBPP, we report pass@10 accuracy and use a sampling temperature of 0.8. 
    \item \textbf{Toxigen (Harmlessness)}~\citep{hartvigsen2022toxigen}: We use the default setup and report the percentage of \textit{non-toxic} model responses. In particular, Toxigen outputs the fraction of toxic responses $f$, and we report $1-f$ expressed as a percentage. 

\end{itemize}

\section{Licenses}

Below, we give all the artifacts that we used and their respective licenses:
\begin{itemize}
    \item Alignment Handbook: Apache-2.0 License
    \item open-instruct: Apaches-2.0 License
    \item Zephyr 7B SFT Full: Apache-2.0 License
    \item IFEval: Eclipse Public License - v 2.0 License
    \item AlpacaEval: Apache 2.0 License
    \item MBPP: MIT License
    \item HumanEval: MIT License
    \item Toxigen: MIT License
    \item GPT-4 outputs: AlpacaEval uses GPT-4 Turbo outputs. Since we use output for research, we are in compliance with the GPT-4 terms of service.
\end{itemize}


\end{document}